\documentclass[twoside]{article}

\usepackage[accepted]{aistats2025}
\usepackage{hyperref}       % hyperlinks
\usepackage{url}            % simple URL typesetting
\usepackage{booktabs}       % professional-quality tables
\usepackage{amsfonts}       % blackboard math symbols
\usepackage{nicefrac}       % compact symbols for 1/2, etc.
\usepackage{microtype}      % microtypography
\usepackage{xcolor}         % colors
\usepackage{algorithm} % Include the algorithm package
\usepackage{algpseudocode} % Include the algorithmic 
\usepackage{float}
\usepackage{graphicx}
\usepackage{subcaption}
\usepackage{float}
\usepackage{arydshln}
\usepackage{xr}

\usepackage{amsmath}
\usepackage{amssymb}
\usepackage{mathtools}
\usepackage{amsthm}

\newtheorem{theorem}{Theorem}[section]
\newtheorem*{theorem*}{Theorem}

\newtheorem{lemma}[theorem]{Lemma}

\newcommand{\bbR}{\mathbb{R}}

\newcommand{\bbE}{\mathbb{E}}

\newcommand{\cX}{\mathcal{X}}
\newcommand{\cR}{\mathcal{R}}

\newcommand{\cF}{\mathcal{F}}
\newcommand{\cG}{\mathcal{G}}

\newcommand{\argmin}{\mathop{\mathrm{argmin}}}

\usepackage[round]{natbib}

\begin{document}



\twocolumn[

\aistatstitle{Offline Stochastic Optimization of Black-Box Objective Functions}

\aistatsauthor{ Juncheng Dong\textsuperscript{1} \And Zihao Wu\textsuperscript{1} \And Hamid Jafarkhani\textsuperscript{2} \And  Ali Pezeshki\textsuperscript{3} \And Vahid Tarokh\textsuperscript{1}}

% \aistatsaddress{ Duke University \And Duke University \And   UC Irvine \And Colorado State University  \And Duke University} ]
\aistatsaddress{} ]
\begin{abstract}
Many challenges in science and engineering, such as drug discovery and communication network design, involve optimizing complex and expensive black-box functions across vast search spaces. Thus, it is essential to leverage existing data to avoid costly active queries of these black-box functions. To this end, while Offline Black-Box Optimization (BBO) is effective for deterministic problems, it may fall short in capturing the stochasticity of real-world scenarios. To address this, we introduce \emph{Stochastic~Offline~BBO} (SOBBO), which tackles both black-box objectives and uncontrolled uncertainties. We propose two solutions: for large-data regimes, a differentiable surrogate allows for gradient-based optimization, while for scarce-data regimes, we directly estimate gradients under conservative field constraints, improving robustness, convergence, and data efficiency. Numerical experiments demonstrate the effectiveness of our approach on both synthetic and real-world tasks.
% Many core challenges in science and engineering, such as drug discovery and the design of communication network topologies, involve optimizing complex and expensive black-box functions within vast search spaces. Therefore, leveraging existing data to avoid actively querying the expensive black-box objective is highly beneficial. Offline black-box optimization (BBO) has proven effective in finding optimal designs under these constraints. However, offline BBO focuses on deterministic optimization problems, which do not account for the inherent stochasticity of real-world scenarios. In this context, we introduce and explore a novel type of optimization, namely \emph{Stochastic Offline BBO} (SOBBO), which addresses the challenges of \emph{out-of-control} uncertainties in addition to the black-box objective function. We propose two different solutions for large and scarce-data regimes. Specifically, for the large-data regime, we propose to estimate the black-box objective function with a differentiable surrogate whose gradient is subsequently used by gradient-based optimization techniques. For the scarce-data regime, our proposed solution is to directly estimate the true gradient while accounting for the requirement that gradients are conservative fields. This approach provides a stronger inductive bias that leads to improved robustness to noise, convergence speed, and data efficiency. Extensive numerical studies demonstrate the efficacy of our proposed approach for various objective functions both for synthetic and real-world scenarios. 
\end{abstract}

\section{Introduction}\label{sec:intro}
% \JD{
% \begin{enumerate}
%     \item BBO
%     \item SO and SOBBO 
%     \item Regular approach for SO
%     \item Recover gradient
% \end{enumerate}
% }
Many core challenges in science and engineering, including drug discovery and the design of communication network topologies, involve optimizing complex black-box functions within vast search spaces. The primary difficulty lies in the fact that evaluating these black-box functions is often costly and time-intensive, e.g., real-world lab experiments are necessary to evaluate the effectiveness of a proposed drug design. It is thus beneficial to utilize existing data so that we can circumvent active query of the expensive black-box objective. 
To this end, offline black-box optimization (BBO) assumes an dataset of historical function evaluations and aim to find an optimized design \textbf{\emph{with only the offline dataset}}~\citep{coms}. Specifically, the goal of offline BBO is to optimize a black-box function $\nu:\Theta \rightarrow \bbR$, that is, to identity $\argmin_{\theta \in \Theta} \nu(\theta)$ with only an offline dataset $\{\theta_i,y_i\}^n_{i=1}$ where $y_i = \nu(\theta_i)$ is the objective value for $\theta_i$. 
% The most intuitive and popular approaches for BBO are \emph{surrogate-based gradient ascent approaches}, which first learn a surrogate model $\widehat \nu$ that approximates the block-box objective $\nu$ using the offline data. After the model is trained, the gradient of $\widehat \nu$ can be used by first-order methods such as Momentum and ADAM to find the optimal design.
Despite its success, offline BBO focuses solely on deterministic optimization problems, which \textbf{\emph{may fall short in capturing the inherent stochasticity of real-world scenarios}}.

\textbf{Stochastic Offline BBO.} In this light, we propose and investigate a new type of optimization, namely \emph{Stochastic Offline BBO} (\textbf{SOBBO}), that is at the intersection of offline BBO and stochastic optimization. In addition to the black-box objective function, SOBBO accounts for \textbf{\emph{out-of-control uncertainties}}, and the goal is find a design that is optimal \textbf{\emph{in expectation}}. To make it more concrete, communication-network design often involves complicated stochasticity such as the change of weather and unexpected signal interference. An optimized design should perform well in expectation across all sorts of weather and interferences. 

Therefore, let $X \in \mathcal{X}$ be a random variable that represents such stochasticity and $\theta$ the design parameter, the goal of SOBBO is to find an optimal $\theta^\star$ such that
\begin{equation}\label{eqn:sobbo-obj}
    \theta^\star \in \argmin_{\theta \in \Theta}\Big\{\nu(\theta) = \bbE_{X}\left[g(\theta,X)\right]\Big\},
\end{equation}
where $g:\Theta \times \mathcal{X} \rightarrow \mathbb{R}$ is the black-box objective function. We will call $\nu(\theta)$ in~\eqref{eqn:sobbo-obj} as the \emph{value function} that we aim to optimize. Here, the distribution of $X$ is invariant to the value of $\theta$, representing that the randomness of $X$ cannot be controlled such as the change of weather in the real world. 

The objective in~\eqref{eqn:sobbo-obj} is reminiscent of the stochastic optimization (SO)~\citep{stochastic_opt_review}. However, the key difference is that \textbf{\emph{the objective function}} $g$ \textbf{\emph{in SO is assumed to be either known or can be easily evaluated}} while $g$ in SOBBO is \textbf{\emph{a block-box function whose evaluation is expensive and time-consuming}}, e.g., complex and extended simulations are inevitable when the performance of a specific network design is under query. 
\begin{figure}[h]
\includegraphics[width=\columnwidth]{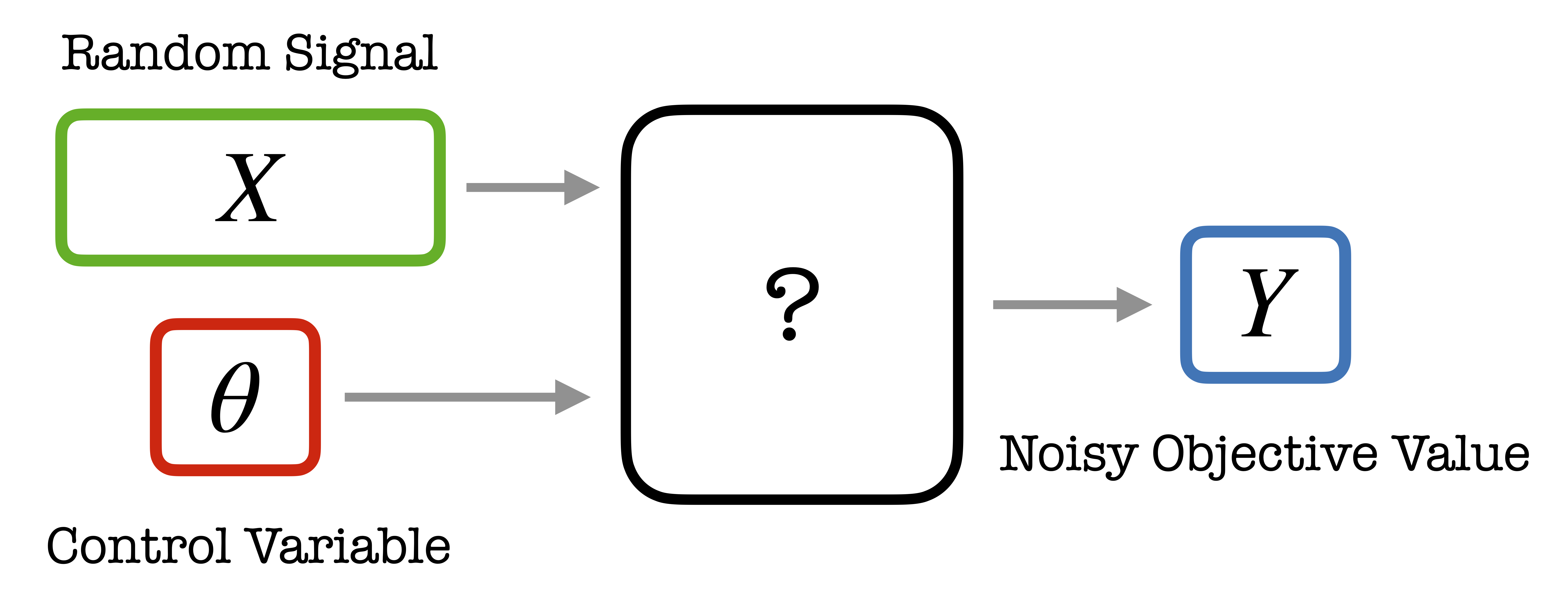}
\caption{This work considers \textbf{\emph{stochastic}} optimization of a \textbf{\emph{black-box}} objective function from only \textbf{\emph{historical data}}.}
\label{fig:prob-setting}
\end{figure}

{\bf Problem Setting.} To facilitate offline learning, SOBBO assumes an offline (historical) dataset $\mathcal{D}=\{\theta_k,x_k,y_k\}^n_{i=1}$, consisting of $n$ independent and identically distributed (iid) samples of the random tuple $(\theta,X,Y)$ where $\theta$ has distribution $P_\theta$, $X$ has law $P_X$, and $Y=g(\theta,X)+\epsilon$ is the (noisy) objective value with $\epsilon \sim\mathcal{N}(0,\sigma^2)$ as the noise which is independent of $\theta$ and $X$. See Figure~\ref{fig:prob-setting} for a visual illustration. With \emph{only} the offline dataset $\mathcal{D}$ and \emph{without} knowledge about $g$, the goal of SOBBO is to find a design $\theta$ whose true value $\nu(\theta)$ is close to the optimal value $\nu(\theta^\star)$. 
For technical soundness, we assume $g$ is bounded and continuously differentiable, allowing the applications of gradient-based optimization techniques; we also assume $\nabla_\theta\bbE_X[g(\theta,X)]$ exists for all $\theta$ so that the gradients can be estimated through data~\citep{renyi2007probability}.
\begin{figure}[h]
    \centering
    \includegraphics[width=\columnwidth]{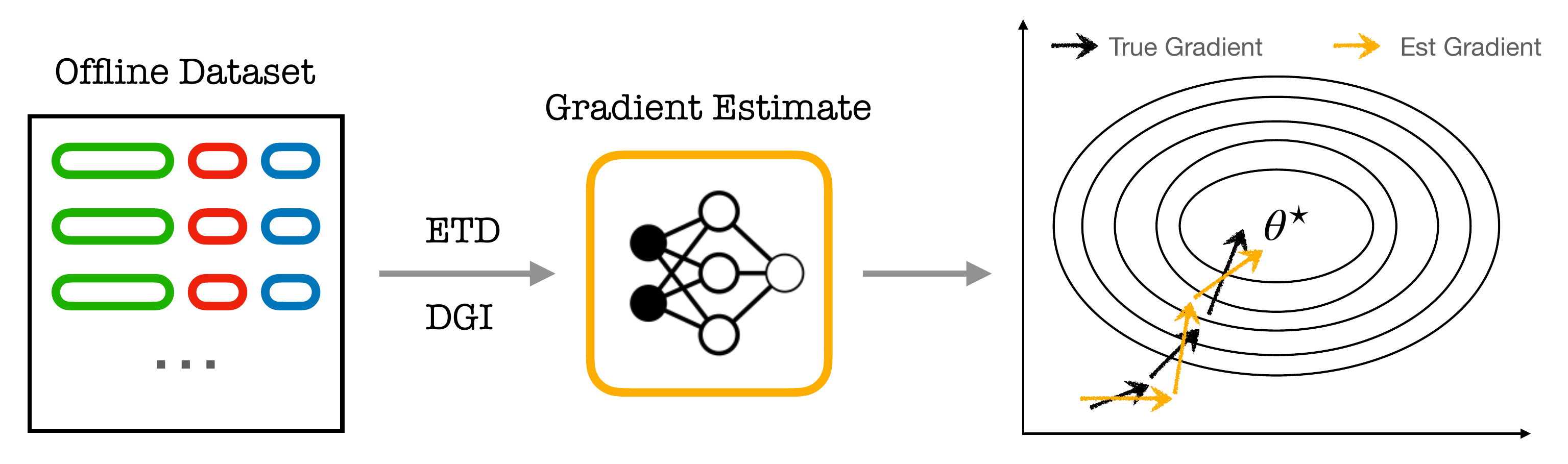}
    \caption{Our proposed methods learn to interpolate the true (unobserved) gradient with an offline dataset, and subsequently generate gradient estimates to be employed by various gradient-based techniques to find the optimal design.}
    \label{fig:pipeline}
\end{figure}

{\bf Motivation.} In the setting of SO where $g$ is known and there exists a set $\mathcal{D}_x=\{x_k\}^n_{k=1}$ containing iid samples of $X$. The most straightforward approach to optimize~\eqref{eqn:sobbo-obj} is through stochastic gradient descent (SGD)~\citep{arora2018convergence} where we approximate the gradient of $\nu$ with 
$$
\nabla_{\theta}\nu_n(\theta)=\frac{1}{n}\sum^n_{k=1}\nabla_\theta g(\theta,x_k).
$$
While the effectiveness of SGD and its variations (e.g., SGD with momentum~\citep{polyak1964some}) has been both theoretically proven and empirically validated in the context of SO, these gradient-based approaches become inaccessible in the setting of SOBBO, where $g$ is a black-box function. To this end, we propose methods that \textbf{\emph{estimate the inaccessible gradient function}} $\nabla_\theta g$ \textbf{\emph{so that all the well-established gradient-based optimization approaches become applicable}}, and existing high-quality code bases for gradient-based methods, including efficiency-optimized auto-differentiation frameworks~\citep{paszke2019pytorch}, can be reused.
See Figure~\ref{fig:pipeline} for a conceptual demonstration of the proposed pipeline.

% {\bf Challenges.} Notably, SOBBO reduces to offline BBO when $X$ in the objective of SOBBO~\eqref{eqn:sobbo-obj} is deterministic, i.e., $\bbP(X=x)=1$ for some $x \in \cX$. Thus, SOBBO not only encompasses offline BBO but also extends it by incorporating stochasticity. Moreover, if for any given $\theta_i$, we can have $y^\star_i = f(\theta_i) = \bbE_X[g(\theta_i,X)]$, then SOBBO also reduces to offline BBO. However, we can only have one sample $g(\theta_i,x_i)$ (with some noise) at a realized $x_i$ of $X$. When active query is allowed, one can attempt to acquire repeated samples of $g(\theta_i,X)$ for a fixed $\theta_i$ to have a reliable estimate of $y^\star_i$. However, this approach is extremely expensive because evaluation of $g$ is costly and, more importantly, it is unlikely to observe the same $\theta_i$ for multiple times in the offline dataset. 

{\bf Contributions.} In particular, we propose two distinct solutions, tailored to different application scenarios of SOBBO based on the volume of available offline data. Specifically, we consider two scenarios:
\begin{itemize}
    \item \textbf{Large-Data Regime:} We propose  \emph{Estimate-Then-Differentiate} (ETD) for applications with opulent data (see Section~\ref{sec:large-data}). ETD first estimates the black-box objective function using a differentiable function, such as a neural network, and subsequently applies its gradient for optimization. ETD is \textbf{\emph{convenient and easy to implement}} with modern deep learning frameworks, leveraging the extensive development tools created by the community to rapidly build solutions for new data-driven problems. Notably, we \textbf{\emph{establish asymptotic theoretical guarantees for this approach}}, supporting its validity in the large-data regime.
    \item \textbf{Scarce-Data Regime:} We propose \emph{Deep Gradient Interpolation} (DGI) for applications with limited data (see Section~\ref{sec:small-data}). DGI has stronger inductive bias that leads to \textbf{\emph{more stable convergence and improved performance under limited data}}. Specifically, DGI involves directly estimating the true gradient while accounting for the requirement that gradients are conservative fields.
\end{itemize}

The effectiveness of the proposed methods is demonstrated through extensive experiments on a wide range of simulation problems and real-world applications (see Section~\ref{sec:exp}). 
% It has been observed that DGI provides stable and superior gradient estimation in both scarce and large-data regimes. Meanwhile, ETD remains a valuable method in the large-data regimes due to its simplicity in construction and computational efficiency.

\section{Related Work}

{\bf Stochastic Optimization (SO) and Offline BBO.} 
This work considers SO without the gradient information of the objective function. 
% SO algorithms can be broadly categorized by whether they require gradient information. 
Derivative-free methods (DFMs), exemplified by the random search algorithms, systematically probe the solution space 
% by iterative refinement of the input variables 
until certain predefined criteria is met~\citep{brooks1958discussion,pronzato1984general}. Evolutionary random search algorithms further integrate mechanisms such as mutation~\citep{holland1992adaptation}, adaptive selection~\citep{reeves1997genetic,whitley1994genetic}, simulated annealing~\citep{bertsimas1993simulated,skiscim1983optimization}, and memory structures~\citep{glover1990tabu, glover1999scatter}. 
% In addition, model-based evolutionary algorithms deduce a probability distribution across potential solutions and generate new candidate solutions from this distribution~\citep{fu2006model,pena2002estimation}. 
While DFMs can be used to optimize unknown objective functions, they require real-time access to the objective values of candidate solutions at each search step, e.g., through a simulator. One research direction close to ours is the Meta-Simulation Optimization (MSO)~\citep{tekin2004simulation}. MSO also addresses the scenarios where the objective functions are unknown and considered as black boxes. However, similar to DFMs, MSO methods also focus on the online setting where the objective functions can be evaluated in real time. 
% delineates a particular class of SO where the objective is to identify input variables that, without loss of generality, minimize the output of a stochastic simulation. MSO operates under the assumption that the simulation behaves as a black box, offering only objective evaluations and constraint assessments for given inputs, without providing an algebraic objective function. Thus, direct gradient measurements are unattainable; however, alternative SO algorithms that do not require such measurements remain viable.
\textbf{\emph{Our work addresses a markedly different and more challenging setting, where only a dataset of historical observations is available}}. 
% Moreover, general theoretical results for DFMs are challenging to establish, leaving their performance and sample complexity under general conditions unwarranted. 
On the other hand, offline BBO algorithms also consider optimization problems of black-box functions with historical data~\citep{coms,mashkaria2023generative,chemingui2024offline}. However, they primarily focus on deterministic problems, whereas SOBBO addresses scenarios with uncontrollable stochasticity. 

{\bf Gradient-based Optimization Algorithms.}  Motivated by the recent tremendous success of gradient-based optimization techniques in data-driven applications~\citep{haji2021comparison, ruder2016overview}, a series of work have established the efficacy of gradient-based methods in non-convex and stochastic optimization problems~\citep{ hu2024convergence, lei2020adaptivity, yazan2017comparison}. 
% Stochastic Approximation (SA)~\citep{dvoretsky1955stochastic}, or gradient-based SO approach, aimed at identifying local minima where gradients vanish. These are commonly referred to as stochastic gradient methods when direct gradient measurements are obtainable. 
Gradient descent remains a foundational technique in training deep neural networks due to its robust convergence guarantees~\citep{arora2018convergence} and computational efficiency~\citep{bottou2010large}. 
% When the objective function is known, stochastic gradient methods~\citep{dvoretsky1955stochastic,ketkar2017stochastic} approximate the true gradients of loss functions with estimations derived from data. 
Representative methods include SGD with momentum~\citep{liu2020improved,polyak1964some}, Adam~\citep{kingma2014adam}, Nesterov's accelerated gradient~\citep{nesterov1983method}, Adagrad~\citep{duchi2011adaptive}, and RMSprop~\citep{hinton2012neural}. 
% which incorporate additional insights from previous iterations to enhance optimization efficiency. 
See~\cite{daoud2023gradient} for a detailed review. However, due to their requirements of the gradient of objective functions, stochastic gradient methods are not directly applicable in the setting of SOBBO. 
% Nevertheless, these algorithms can still be effectively utilized if additional gradient estimation techniques are employed. Therefore, 
To this end, \textbf{\emph{our work aims to develop methods that effectively extract gradient information from pre-established datasets so that existing gradient-based methods can be directly reused}}. 

{\bf Gradient Estimation.} Existing gradient estimation techniques typically follow finite-difference schemes, such as Perturbation Analysis or Likelihood Ratio/Score Function, to approximate the gradient~\citep{glasserman1990gradient,ho2012perturbation,rubinstein1993discrete}. Simultaneous perturbation stochastic approximation (SPSA)~\citep{spall1992multivariate} achieves superior results with just two sample points. More recent advancements in SPSA include integrating global search features by introducing Monte Carlo perturbations during updates~\citep{maryak2001global,wang2011discrete} and refining estimates of Jacobian and Hessian matrices~\citep{spall2009feedback}. However, all these algorithms require real-time access to objective functions, imposing considerable constraints on the optimization environment and incurring high simulation costs. \textbf{\emph{Our work overcomes these drawbacks by solely utilizing historical data}}.

\section{Large-Data Regime: Estimate-Then-Differentiate (ETD)}\label{sec:large-data}
As motivated in Section~\ref{sec:intro}, when the black-box objective function $g$ is known, we can solve the objective of SOBBO in~\eqref{eqn:sobbo-obj} by estimating the gradient of the true value function $\nabla_\theta\nu$ through \emph{sampling average approximation} (SAA)~\citep{kim2015guide}:
\begin{equation}\label{eqn:saa-gradient}
    \nabla_\theta\nu_n(\theta) = \frac{1}{n}\sum_{k=1}^n \nabla_\theta g(\theta,x_k).
\end{equation}
Under mild conditions $\nabla_\theta\nu_n(\theta)$ converges almost surely to the true gradient $\nabla_\theta \nu(\theta)$ for any $\theta \in \Theta$ by the \emph{Strong Law of Large Numbers}~\citep{renyi2007probability}. However, when $g$ is a black-box function, $\nabla_\theta g$ becomes inaccessible thus the approximator in~\eqref{eqn:saa-gradient} cannot be applied. 

{\bf Method.} To this end, a series of preceding work have established that over-parameterized neural networks are effective in interpolating functions in \emph{large-data regime}. Motivated by this, we propose to \textbf{\emph{(i)}} first estimate the true black-box objective function $g(\theta,x): \Theta\times\cX \rightarrow \bbR$ with a differentiable estimator $g_\phi(\theta,x):\Theta\times\cX \rightarrow \bbR$ where $\phi$ is its parameter, and \textbf{\emph{(ii)}} subsequently use its derivative $\nabla_\theta g_\phi$ for gradient-based optimization. We term our method \emph{Estimate-then-Differentiate} (\textbf{ETD}). 

{\bf Training.} Specifically, during training, we apply \emph{Empirical Risk Minimization} (ERM) with the given offline dataset $\mathcal{D}$ to obtain the estimator $g_{\phi_n}$ where 
\begin{equation}\label{eqn:obj-erm}
     \phi_n \in \argmin_{\phi} \left\{L_{\text{ERM}}(\phi)=\frac{1}{n}\sum_{k=1}^n\left(g_\phi(\theta_k,x_k)-y_k\right)^2\right\}.
\end{equation}

{\bf Inference.} During inference, i.e., to solve Problem~\eqref{eqn:sobbo-obj}, 
we apply the trained estimator $g_{\phi_n}$, along with the samples $\mathcal{D}_x=\{x_k\}^n_{k=1}$ in the training dataset (we assume no access to extra data), to estimate the true gradient $\nabla_\theta\nu(\theta)$ via:
\begin{equation}\label{eqn:etd-grad-est}
    \widehat{\nabla_\theta\nu_n}(\theta) := \frac{1}{n}\sum_{k=1}^n \nabla_\theta g_{\phi_n}(\theta,x_k).
\end{equation}
With the estimated gradient $ \widehat{\nabla_\theta\nu_n}(\theta)$, any gradient-based optimization technique, appropriate for the target problem, can be chosen to solve SOBBO. In this our experiments, we choose $g_\phi$ to be deep neural networks (DNNs). This approach is convenient with all the auto-differentiation software frameworks designed for deep learning. 

{\bf Theoretical Guarantee.} Notably, if $\widehat{\nabla_\theta\nu_n}(\theta)$ is a consistent estimator of the true gradient $\nabla_\theta \nu(\theta)$, then we reduce the problem of SOBBO to the canonical setting of SO. Thus, all the established guarantees about stochastic optimization with gradient-based techniques can help establish the soundness of this method~\citep{shapiro1996convergence,arora2018convergence}. 
% While there has been a considerable amount of previous work on the consistency of DNNs in estimating the target function, we next establish the gradient of the learned estimator is also consistent.
We next establish that \textbf{\emph{the proposed gradient estimator of ETD in~\eqref{eqn:etd-grad-est} is indeed consistent}}.

\begin{theorem}\label{thm:consistency-of-gradient}
Let $\cG$ be a function class of bounded continuously differentiable functions such that $g \in \cG$. Assume that $\Theta$ and $\cX$ are compact. Let $\widehat g_n \in \cG$ be the solution to the Empirical Risk Minimization problem defined in~\eqref{eqn:obj-erm}, where $n$ is the total number of training samples in the offline dataset. Let 
$$
\eta(\theta) = \nabla_{\theta}\bbE_X[g(\theta,X)]\;\;\text{and}\;\;\widehat{\eta}_n(\theta) = \frac{1}{n}\sum^n_{i=1}\nabla_{\theta}\widehat g_n(\theta,X_i) 
$$
respectively denote the true gradient and the gradient estimate using $\widehat g_n$ and the offline dataset $\mathcal{D}$, then 
\begin{align*}\bbE_\theta\bigg[\big\|\eta(\theta)-\widehat{\eta}_n(\theta)\big\|\bigg]  \stackrel{P}{\rightarrow} 0,
\end{align*}
where $\stackrel{P}{\rightarrow}$ denotes convergence in probability. 
\end{theorem}
\begin{proof}[Proof of Theorem~\ref{thm:consistency-of-gradient}] Please see in Appendix~\ref{sec:app-proof}.\end{proof}

Theorem~\ref{thm:consistency-of-gradient} suggests that,\textbf{\emph{with sufficient data}}, ETD is guaranteed to generate reliable estimation of the true gradient. In particular, \textbf{\emph{the average estimation error can be arbitrarily small}} with a sufficiently large offline dataset. We will \textbf{\emph{empirically validate this on both simulation and real-world applications}} in Section~\ref{sec:exp}. Notably, this result remains valid when ETD uses the data $\mathcal{D}_x$ in the offline dataset to estimate the gradient, avoiding the need for sample splitting~\citep{hansen2000sample}.

\section{Scarce-Data Regime: Deep Gradient Interpolation (DGI)}\label{sec:small-data}
In many practical cases, there may not be enough data to meet the requirements of the asymptotic guarantees in Section~\ref{sec:large-data}. In these scarce-data scenarios, it is beneficial to incorporate more inductive bias to aid learning. To this end, we present \emph{Deep Gradient Interpolation} (\textbf{DGI}),  which directly interpolates the true gradient from the offline dataset. 

Specifically, consider a learner (e.g., neural network) $h_\phi(\theta,x): \Theta \times \mathcal{X} \rightarrow \bbR^d$ where $\phi$ denotes its parameters. The objective of $h_\phi$ is to estimate $\nabla g$, that is, $h_\phi(\theta,x) \approx \nabla_{\theta,x}g(\theta,x)$. Motivated by the \emph{Fundamental Law of Calculus} ~\citep{sobczyk2011fundamental} and the \emph{Clairaut's Theorem}~\citep{aso1991generalization}, we propose the following learning principles for $h_\phi$:
\begin{itemize}
    \item[(I)] \textbf{\emph{Balance Equations}}. $h_\phi(\theta,x)$ outputs the gradient of a function, thus it must be a conservative field and satisfy a set of balance equations. See Section~\ref{sec:loss1}. 
    \item[(II)] \textbf{\emph{Reconstruction}}. Path integral of $h_\phi(\theta,x)$ must estimate the difference between values of $g$ at two ending points. See Section~\ref{sec:loss2}.
    \item[(III)] \textbf{\emph{Path Independence}}. Given that $h_\phi$ is a conservative field, the values of the path integrals of $h_\phi$ are independent of the chosen paths. See Section~\ref{sec:loss3}.
\end{itemize}

Notably, in our experimental study, we verify that all three training objectives significantly contribute to the performance. Next, we elaborate on each of the principles in sequence. 
% Algorithm~\ref{algo:dgx} (in the appendix) presents the pseudocode of the proposed algorithm.

\subsection{Balance Equations}\label{sec:loss1}
Let $\Theta \in \bbR^{d_\theta}, \cX \in \bbR^{d_x}$. For notational simplicity, let $\zeta = (\theta,x) \in \bbR^d$ where $d=d_\theta+d_x$. By the Clairaut's Theorem, if $h_\phi(\zeta)$ is the gradient of a second-order continuously differentiable function, then for any $x \in \mathcal{X}$, $\theta \in \Theta$, it needs to satisfy the following equation(s):
\begin{equation}\label{eqn:eqn_1}
\begin{aligned}
     & \frac{\partial h_\phi^{j}(\zeta)}{\partial \zeta^i} = \frac{\partial h_\phi^i(\zeta)}{\partial \zeta^j}; \quad\text{ for all $i,j \in [d]$},\\
\end{aligned}
\end{equation}
where $h_\phi^j$ is the $j$-th output of $h_\phi$, and $\zeta^i$ is the $i$-th element of $\zeta$. To motivate $h_\phi$ to satisfy~\eqref{eqn:eqn_1}, we propose optimizing the following data-dependent surrogate loss:
\begin{equation}
    L_b(\phi) = \bbE_{\zeta}\bbE_{i,j}\left(\frac{\partial h_\phi^j(\zeta)}{\partial \zeta^i} -\frac{\partial h_\phi^i(\zeta)}{\partial \zeta^j}\right)^2,
\end{equation}
where $i,j$ in the expectation follow a uniform distribution on the finite set $\{1,\dots,d\}$. Note that for any $\phi$ such that $L_b(\phi)=0$, we have $h_\phi$ satisfy the balance equations~\eqref{eqn:eqn_1} almost surely. $L_b$ can be unbiasedly estimated through the given offline dataset $\mathcal{D}$ as: 
\begin{equation}\label{eqn:balance-loss-dataset}
    \widehat{L_b}(\phi) = \frac{1}{n}\sum_{k=1}^n\bbE_{i,j}\left(\frac{\partial h_\phi^j(\zeta_k)}{\partial \zeta^i} -\frac{\partial h_\phi^i(\zeta_k)}{\partial \zeta^j}\right)^2,
\end{equation}
where $\zeta_k = (\theta_k,x_k) \in D$.

\subsection{Reconstruction}\label{sec:loss2}
By the Fundamental Law of Calculus, as $h_\phi$ estimates $\nabla g$, the integral of $h_\phi(\theta,x)$ should estimate changes of $g$. Specifically, for any two pairs of $(\theta_1,x_1), (\theta_2,x_2) \in \Theta\times\cX$, we should have 
\begin{equation}\label{eqn:integral-req}
\begin{aligned}
g(\theta_1,x_1)-g(\theta_2,x_2) &= \int^{\theta_1,x_1}_{\theta_2,x_2}h_\phi(\Tilde{\theta},\Tilde x)d\Tilde{\theta}d\Tilde x,
\end{aligned}
 \end{equation}
where the integral is the path integral from $ (\theta_2,x_2)$ to $(\theta_1,x_1)$. Motivated by this, we propose the following training objective function for $h_\phi$:
\begin{equation*}
% \label{eqn:recon_loss_data}
    L_r(\phi) = \bbE_{Z_1,Z_2}\left[\left(Y_1-Y_2-\int_{\theta_2,X_2}^{\theta_1,X_1}h_\phi(\Tilde{\theta},\Tilde x)d\Tilde{\theta}d\Tilde x\right)^2\right],
\end{equation*}
where $Z_1=(\theta_1,X_1,Y_1)$ and $Z_2=(\theta_2,X_2,Y_2)$ are independent tuples both following the data distribution described in Section~\ref{sec:intro}. Note that the integral of $h_\phi$ in $L_r(\phi)$ is \emph{not path-independent} unless it is a conservative field. To let the integral be well-defined, for now, we assume the linear path from $ (\theta_2,x_2)$ to $(\theta_1,x_1)$ as the integral's path. In the next section, we will consider more complex path schemes for learning.  The soundness of $L_r$ is due to the fact that minimizers of $L_r$ satisfy the integral requirement in~\eqref{eqn:integral-req} almost surely, as demonstrated by the following result. 
% that can be easily established and we include for completeness.
\begin{lemma}\label{lemma:validity-of-recon-loss}
Consider the data generation process $Y = g(\theta,X) + \epsilon$ where $\epsilon$ is independent of $(\theta,X)$ and $\bbE[\epsilon]=0,\mathrm{Var}[\epsilon^2]=\sigma^2$. 
% Let $L_r$ be defined as in~\eqref{eqn:recon_loss_data}. 
Then we have $\min_{\phi'}L_r(\phi')=2\sigma^2$. Moreover, if $L_r(\phi)=2\sigma^2$, then for any two random pairs of $(\theta_1,x_1), (\theta_2,x_2) \in \Theta\times\cX$ following the data distribution, 
\begin{align*}
   g(\theta_1,x_1)-g(\theta_2,x_2) = \int_{\theta_2,x_2}^{\theta_1,x_1}\nabla_{\theta,x}h_\phi(\Tilde{\theta},\Tilde x)d\Tilde{\theta} d\Tilde x
\end{align*}
almost surely.
\end{lemma}
\begin{proof}[Proof of Lemma~\ref{lemma:validity-of-recon-loss}]Please see in Appendix~\ref{sec:app-proof}.\end{proof}

Given a dataset $\mathcal{D}=\{\theta_k,x_k,y_k\}^n_{k=1}$ which consists of iid tuples of $(\theta,X,Y)$, $L_r(\phi)$ can be unbiasedly estimated as
\begin{equation*}
% \label{eqn:recon-loss-dataset}
    \widehat{L_r}(\phi) = \frac{1}{n^2}\sum_{k,k' \in [n]}\left(y_k-y_{k'}-\int^{\theta_{k'},x_{k'}}_{\theta_k,x_k}h_\phi(\Tilde{\theta},\Tilde x)d\Tilde{\theta}d\Tilde x\right)^2.
\end{equation*}

\subsection{Path Independence}\label{sec:loss3}
As $h_\phi$ must be a conservative field, the value of the path integral in~\eqref{eqn:integral-req} should be independent of the selected path. Specifically, for any pair of $\zeta_1 = (\theta_1,x_1), \zeta_2 = (\theta_2,x_2) \in \Theta\times\cX$, let $\rho(\zeta_1,\zeta_2)$ denote a set of smooth curves with end points $\zeta_1$ and $\zeta_2$. Then, the path integral of $h_\phi$ along any path $r \in \rho$ should generate the same result. To this end, we propose the following objective function to assist $h_\phi$ to better capture this desired behavior: given the offline dataset $\mathcal{D}=\{\zeta_k=(x_k,\theta_k),y_k\}_{k=1}^n$,
\begin{equation}\label{eqn:recon-loss-sup}
    L_e(\phi) = \frac{1}{n^2}\sum_{k,k' \in [n]} \sup_{r \in \rho(\zeta_k,\zeta_{k'})}\left( y_k-y_{k'}-\int_{r }h_\phi(\zeta)d\zeta \right)^2.
\end{equation}
Here, compared to $L_r(\phi)$, we consider the \emph{worst-case} estimation error over the set of all smooth paths in $\rho$. Such an objective will enforce $h_\phi$ to optimize over all paths, leading to integral results that are close to path-independent. As $\sup$ over the infinite dimensional space $\rho$ is intractable, in every training step (update of $\phi$), we randomly sample $\Lambda$ paths $\{r^{k,k'}_\lambda\}^\Lambda_{\lambda=1}$ for every path set $\rho(\zeta_k,\zeta_{k'})$ in~\eqref{eqn:recon-loss-sup} and optimize 
% \begin{equation}\label{eqn:obj-discrete}
% \widehat L_e(\phi) = \sum_{i=1}^n \max_{r \in \{r_k\}}\left( (y_i-y'_i)-\int_{r }h_\phi(x_i,\theta)d\theta\right)^2.
% \end{equation}
\begin{equation}\label{eqn:recon-loss-sup-discrete}
    \widehat{L_e}(\phi) = \frac{1}{n^2}\sum_{k,k' \in [n]} \max_{r \in \{r^{k,k'}_\lambda\}^\Lambda_{\lambda=1}}\left( (y_k-y_{k'})-\int_{r }h_\phi(\zeta)d\zeta \right)^2.
\end{equation}

\textbf{Path Parametrization and Sampling.}
The success of the proposed algorithm relies on a careful choice of path space $\rho$ as well as the sampling scheme used to discretize the optimization problem. In this work, we choose $\rho$ to be the space of paths parameterized by polynomials as any smooth trajectory (which is an analytical function) can be approximated arbitrarily well with the class of polynomial functions. Please refer to Appendix~\ref{sec:path-param} for details.
% \textbf{Algorithm Steps}: 
% \begin{enumerate}
%     \item Sample a data point $x_i,\theta_i,\theta'_i,y_i,y'_i$ from the training dataset, sample $K=64$ paths from the path space $\rho$. For all the sampled path $r$, randomly sample the highest degree for $r$, set $r(0) = \theta_i$ and $r(1) = \theta'_i$ and then randomly sample coefficients for $r$.
%     \item Evaluate the following loss
%     $$
%     \widehat L_e(\phi,i)=\max_{r \in \{r_k\}}\left( (y_i-y'_i)-\int_{r }h_\phi(x_i,\theta)d\theta\right)^2.
%     $$
%     \item Optimize $\widehat L_e(\phi,i)$ by gradient descent.
%     \item Repeat 1-3. This can be done in batch manner, i.e., optimize $\sum_{i \in B}\widehat L_e(\phi,i)$ where $B$ is the set of indices for a sampled batch. 
% \end{enumerate}
\subsection{Deep Gradient Interpolation (DGI)}
Here, we provide a brief overview of the proposed algorithm DGI whose pseudocode is presented in Algorithm~\ref{algo:dgx} (in the appendix). 

{\bf Training.} Given the offline dataset $\{x_k,\theta_k,y_k\}^n_{k=1}$ and a neural network $h_\phi: \Theta\times\cX\rightarrow \bbR^d$, DGI optimizes the following objective for $h_\phi$: 
$$\phi_n \in \argmin_{\phi}L(\phi) := \widehat{L_e}(\phi) + \alpha\cdot\widehat{L_b}(\phi),$$
where $\widehat{L_e}(\phi)$ and $\widehat{L_b}(\phi)$ are respectively defined in~\eqref{eqn:recon-loss-sup-discrete} and~\eqref{eqn:balance-loss-dataset}, balanced by the hyper-parameter $\alpha \ge 0$. 

{\bf Inference.} During inference, with the samples of $X$ in the offline dataset $\mathcal{D}_x = \{x_k\}^n_{k=1}$, the gradient of the true loss function $\nu(\theta)$ can be estimated as
\begin{equation}
    \widehat{\nabla_{\theta}\nu_n}(\theta) = \frac{1}{n}\sum_{k=1}^{n} h^{[\theta]}_{\phi_n}(x_i,\phi),
\end{equation}
where $h^{[\theta]}_{\phi_n}$ is the output of the first $d_\theta$ dimensions of $h_{\phi_n}$ that represent the estimation of  $\nabla_{\theta}g(\theta,x)$.

\section{Experiments}\label{sec:exp}
We empirically validate the efficacy of our proposed methods for both large and scarce-data regimes through experiments on a broad range of simulation problems and real-world applications. Specifically, we aim to answer the following questions:
\begin{enumerate}
    \item Can our proposed methods (ETD and DGI) improve over random search and the best design in the offline dataset? 
    (Section~\ref{sec:exp-perf})
    \item Can ETD reliably estimate the gradient in large-data regime? Can DGI provide improved gradient estimation over ETD in scarce-data regime? Are all components of DGI helpful?
    (Section~\ref{sec:exp-grad})
    \item How do the various hyperparameters of DGI impact its performance? How much computational resource does DGI require?
    (Section~\ref{sec:exp-ablation})
\end{enumerate}

In the sequel, \textbf{\emph{we provide affirmative answers to all the above questions}}, and corroborate that \textbf{\emph{DGI is both computationally efficient and robust to hyperparameter choices}}. Due to space limitations, we present a selected subset of results here and defer the full set of results to  Appendix~\ref{sec:app-full-results}. First, we provide a brief overview of the optimization tasks and evaluation protocols used in this section.

\textbf{Optimization Tasks.} The objective functions used in the \textbf{\emph{simulation problems}} are meticulously selected to encompass a broad spectrum of functions with diverse physical properties and shapes. Specifically, in addition to the commonly used \emph{linear} and \emph{quadratic functions} for evaluation,  we include \emph{functions with multiple local minima} (e.g., Ackley, Griewank), \emph{bowl-shaped functions} (e.g., Trid, Perm), \emph{plate-shaped functions} (e.g., Zakharov), and \emph{valley-shaped functions} (e.g., Dixon-Price, Rosenbrock). Moreover, considering neural networks as universal approximators~\citep{scarselli1998universal}, we include \emph{objective functions represented by neural networks of varying scales}, referred to as \emph{NN small} and \emph{NN Large} in the following sections. 
In terms of \textbf{\emph{real-world applications}}, we first focus on two widely studied physics-related scenarios: the welded beam design problem~\citep{deb1991optimal} and the pressure vessel design problem~\citep{moss2004pressure}. Furthermore, we extend our studies to three additional tasks: the continuous newsvendor model~\citep{khouja1999single}, the M/M/1 queue system~\citep{cheng1999improved}, and the stochastic activity network~\citep{avramidis1996integrated}. See Appendix~\ref{sec:data-gen} for details of the aforementioned functions and applications. 
% Without loss of generality, we generate the offline datasets for simulation problems by randomly sampling $\theta$ and $X$ from a three-dimensional unit hypercube following uniform distributions, and compute the corresponding objective value. A zero-mean Gaussian noise is subsequently added. 
% The offline data distributions for real-world applications are more complex and they are detailed in Appendix~\ref{sec:real-setting} due to space constraint. 
Each training dataset in the scarce-data regime contains $128$ samples, whereas those for the large-data regime may contain up to $50$k samples. Due to space constraint, see Appendix~\ref{sec:data-gen} for details on the data generating process. 
% In designing the experiments, we prioritize fairness in comparing ETD and DGI to understand the behavior of the two proposed methods rather than solely focusing on performance. Consequently, both methods utilize the same fully connected layers as backbone and identical hyperparameters. (Refer to the appendix for detailed experiment parameters.) The full potential of our methods can be realized only through an additional comprehensive hyperparameter tuning.

\begin{table}[h]
\caption{Performance for Scarce-Data Regime. \textbf{\emph{Lower}} value indicates \textbf{\emph{better}} performance.} \label{tab:scarce-main}
\begin{center}
\begin{tabular}{lllll}
\textbf{}  &\textbf{Zak} &\textbf{Perm} &\textbf{Rose} &\textbf{Trid}\\
\hline 
% SGDRS&0.0161&0.0303&0.089&0.174 \\
% RS  &0.0228&0.1104&0.202&0.324 \\
% OC  &0.0165&0.0457&0.105&0.212 \\
% ETD(R) &0.0196&0.0495 & 0.160&0.192\\
% ETD(G)  &0.0196&0.0492 & 0.165&0.193\\
% DGI(R) &0.0173&0.0402&0.138&0.192\\
% DGI(G) &0.0173&0.0402&0.138&0.192\\

RS  &0.4161&2.6436&1.2697&0.8621 \\
\hdashline
ETD(R) &0.2174&0.6337&0.7978&0.1034 \\
ETD(G)  &0.2174&0.6238&0.8539&0.1092 \\
DGI(R) &0.0745&0.3267&0.5506&0.1034 \\
DGI(G) &0.0745&0.3267&0.5506&0.1034 \\
\hdashline
OC  &0.0248&0.5083&0.1798&0.2184 \\

% \textbf{}  &\textbf{Rose} &\textbf{NN-l} &\textbf{Trid} &\textbf{Dix} \\

% RS(ood)        & 0.202& 0.489&0.324& 0.220\\
% RS             &0.203 &0.489 & 0.299&0.219\\
% OC           &0.105 & 0.387&0.212 &0.165\\
% True sgd (RS) &0.089&0.260&0.174&0.153\\
% True sgd (greedy)&0.0891&0.263&0.174&0.153\\
% ETD(RS)            &0.160 & 0.472&0.192&0.163\\
% ETD(greedy)             &0.165 &0.490 &0.193&0.164\\
% DGI(RS)&0.138&0.491 &0.192&0.163\\
% DGI(greedy) &0.138& 0.491&0.192&0.163

% \textbf{}  &\textbf{Linear} &\textbf{NN-s} &\textbf{Zak} &\textbf{Perm} \\

% RS(ood)        & 0.518&0.515 &0.0228& 0.1104\\
% RS             & 0.518&0.528 & 0.0240&0.1067\\
% OC           &0.411 &0.385&0.0165 &0.0457\\
% True SGD(RS) &0.295&0.247&0.0161&0.0303\\
% True SGD(greedy) &0.295&0.249&0.0161&0.0303\\
% ETD(RS)            &0.504 &0.496 &0.0196&0.0495\\
% ETD(greedy)             & 0.500&0.479 &0.0196&0.0492\\
% DGI(RS)&0.462 & 0.445&0.0173&0.0402\\
% DGI(greedy) & 0.460&0.441&0.0173&0.0402

\end{tabular}
\end{center}
\end{table}
\begin{table}[h]
\caption{Performance for Large-Data Regime. Tasks include Quadratic (\textbf{\emph{QR}}), Ackley (\textbf{\emph{AK}}), Beam (\textbf{\emph{BM}}), and Vessel (\textbf{\emph{VL}}). \textbf{\emph{Lower}} value indicates \textbf{\emph{better}} performance.} 
\label{tab:large-main}
\begin{center}
\begin{tabular}{rrrrrr}
\textbf{}  &\textbf{QR}  &\textbf{AK}&\textbf{BM} &\textbf{VL} \\
\hline 
% SGDRS        &0     &0.541&-0.0005&0\\
% RS           &0.0815&0.761&0.215&0.185\\
% OC           &0.00071&0.380&0.003&0.0001\\
% ETD(R)      &0.0208 &0.558&0.038&0.0275\\
% ETD(G)       & 0.0191&0.382&0.042&0.0264\\
% DGI(R)      & 0.0032&0.446&-0.0005&0.0219\\
% DGI(G)       &0.0027&0.375&-0.0005&0.0222

RS  &814.0&0.4067&431.0&1849 \\
\hdashline
ETD(R) &207.0&0.0314&77.0&274.0 \\
ETD(G)  &190.0&-0.2939&85.0&263.0 \\
DGI(R) &31.0&-0.1756&0.0&218.0 \\
DGI(G) &26.0&-0.3068&0.0&221.0 \\
\hdashline
OC  &6.1&-0.2976&7.0&0.0 \\

\end{tabular}
\end{center}
\end{table}

\textbf{Evaluation Protocols.} Experiments for the scarce-data regime are repeated $50$ times with varying randomly sampled training datasets. For the large-data regime, $20$ random experiments are conducted. The results are averaged across datasets. As the black-box objectives functions are known for all the optimization tasks in consideration, the true objective values and gradients can be accurately estimated through computational approaches (Appendix~\ref{sec:exp-detail}).
% To evaluate the estimated gradient, we directly backpropagate the objective value on the decision variable 
% $\theta$ over a fixed, randomly sampled set of $x's$, obtaining the estimated true gradient.
\begin{figure*}[!h]
\label{fig:sim-large1}
\centering
\begin{subfigure}[b]{0.6\textwidth}
   \includegraphics[width=\textwidth]{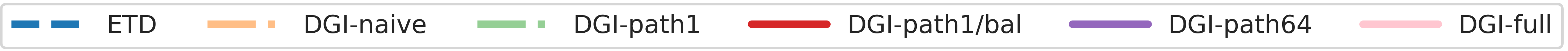}
    \label{fig:sub1}
    \vspace{-1em}
\end{subfigure}
\hfill
\\
% Row 1
% \begin{subfigure}[b]{0.8\textwidth}
%      \includegraphics[width=\textwidth]
%     {NeurIPS2024/images3/Row1_cut_2.png}
%     \label{fig:sub1}
%     \vspace{-1em}
% \end{subfigure}
% \hfill
% \\
\begin{subfigure}[b]{0.8\textwidth}
     \includegraphics[width=\textwidth]
    {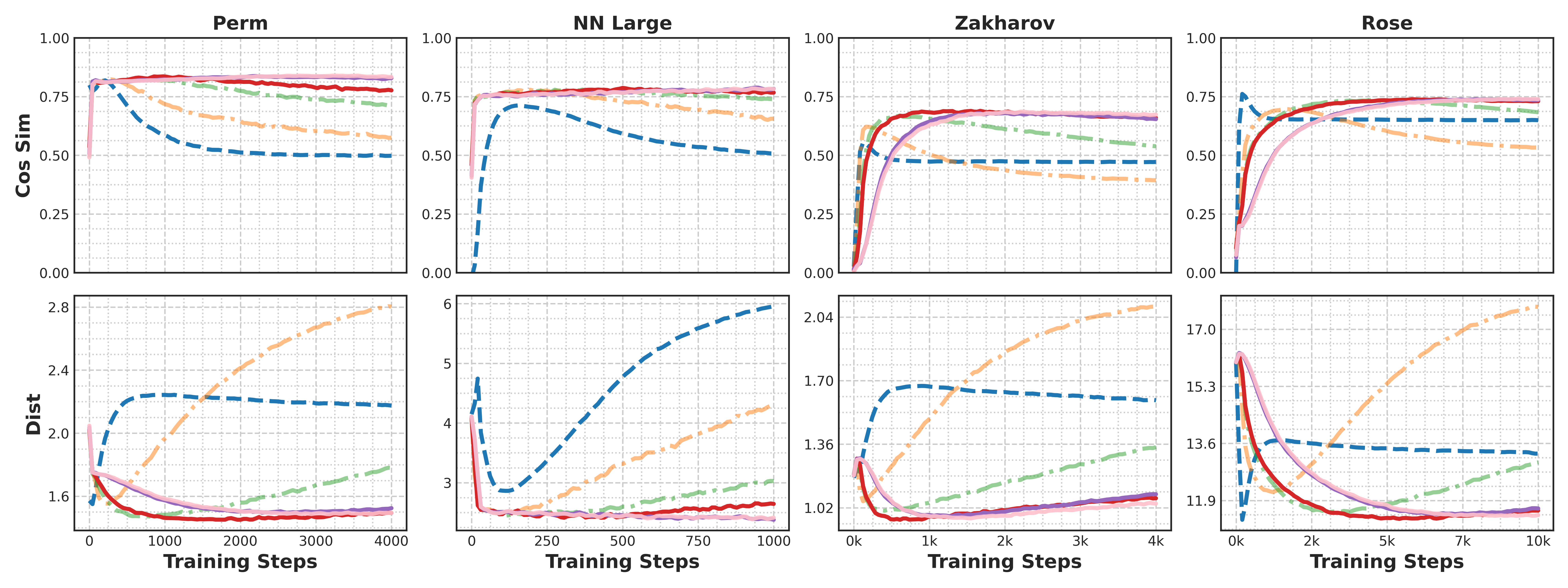}
    \label{fig:sub1}
\end{subfigure}
\hfill
\caption{Gradient Estimation for Scarce-Data Regime. Top figures show cosine similarity (\textbf{\emph{higher}} is better); gradient norm distance (\textbf{\emph{lower}} is better) in bottom figures.}
\label{fig:sim-scarce}
\end{figure*}
\begin{figure}[!h]
\centering
\begin{subfigure}[b]{0.33\textwidth}
    \includegraphics[width=\textwidth]
    {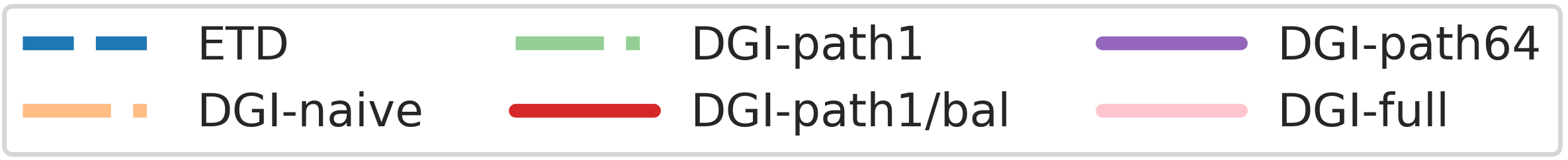}
    % {NeurIPS2024/images3/legend.png}
    \label{fig:sub1}
    \vspace{-1em}
\end{subfigure}
\hfill
\\
% Row 1
\begin{subfigure}[b]{0.36\textwidth}
    \includegraphics[width=\textwidth]{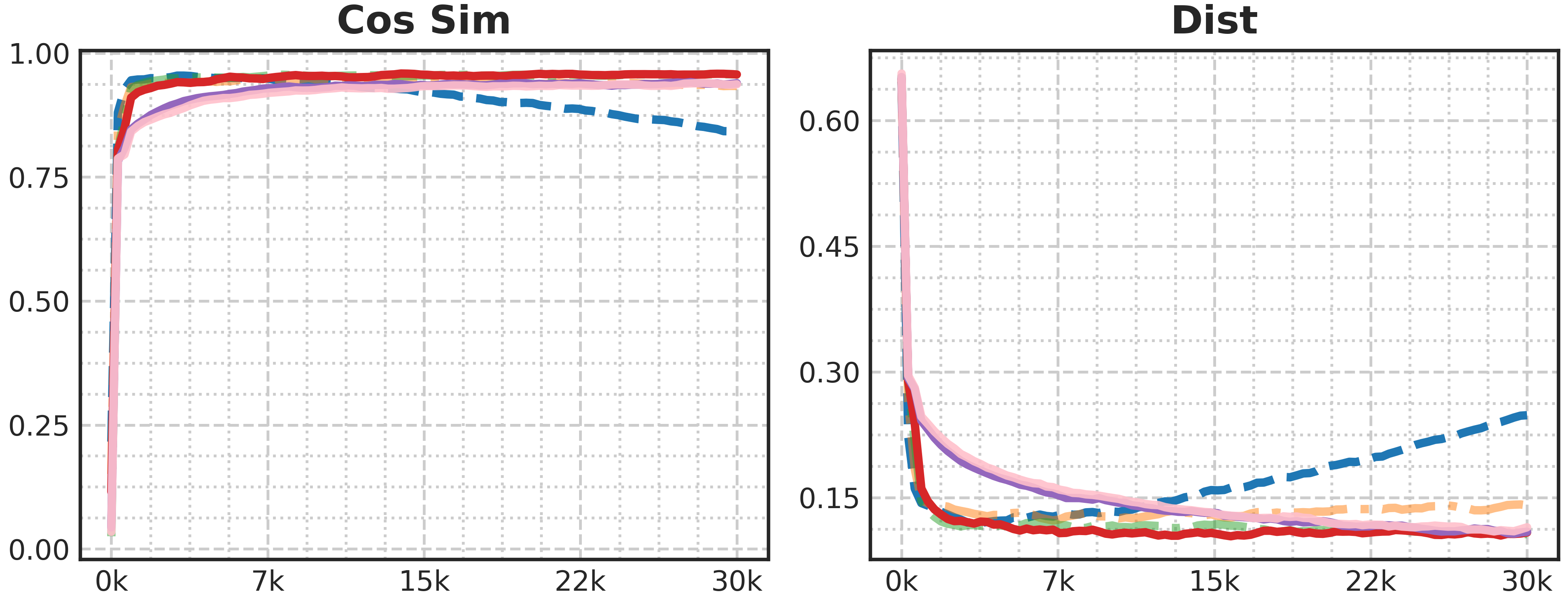}
    \label{fig:sub1}
    \vspace{-1em}
\end{subfigure}
\hfill
% \begin{subfigure}[b]{0.36\textwidth}
%     \includegraphics[width=\textwidth]{AISTATS2025/images/Ackley.png}
%     \label{fig:sub1}
%     \vspace{-1em}
% \end{subfigure}
% \hfill
\begin{subfigure}[b]{0.36\textwidth}
    \includegraphics[width=\textwidth]{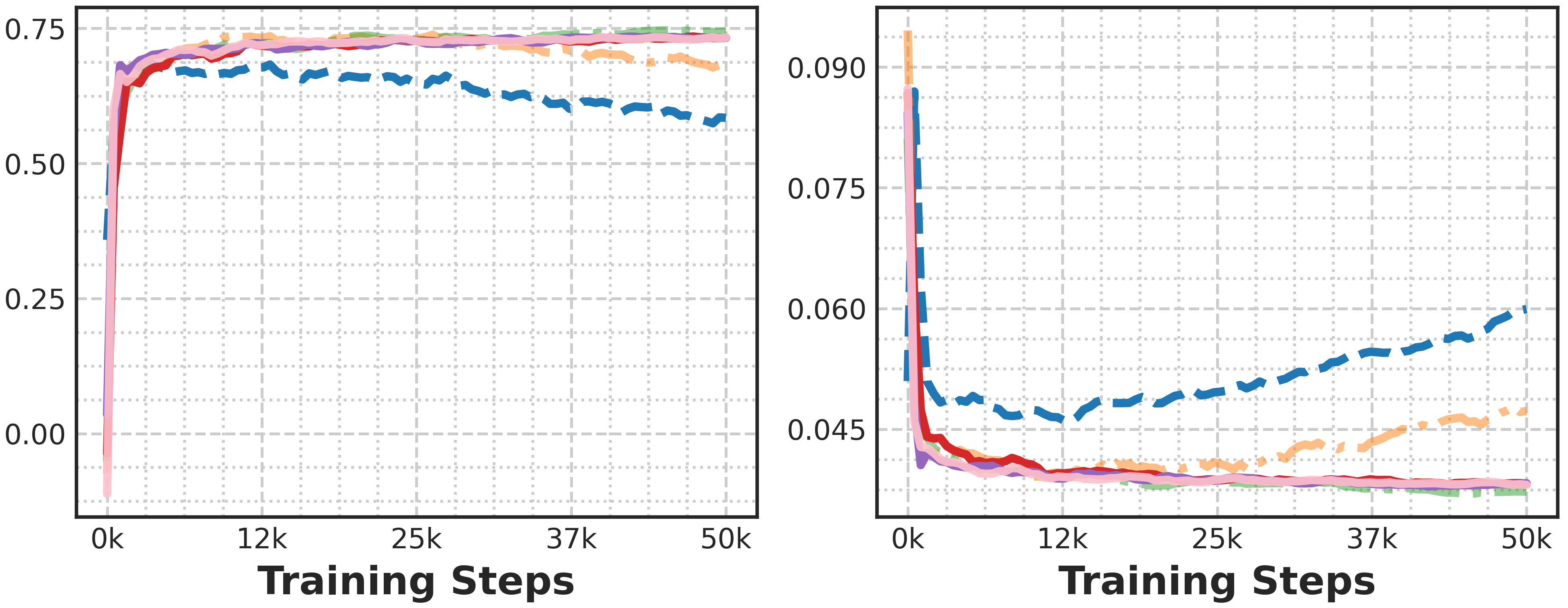}
    \label{fig:sub1}
\end{subfigure}
% \vspace{-1em}
\caption{Gradient Estimation for Large-Data Regime. Cosine similarity and norm distance on \textbf{\emph{Quadratic}} (top), 
% \textbf{\emph{Ackley}} (middle), 
and \textbf{\emph{Griewank}} (bottom) objective functions.}
\label{fig:sim-large1}
\end{figure}
\begin{figure*}[!h]
\centering
% Row 1
\begin{subfigure}[b]{0.32\textwidth}
    \includegraphics[width=\textwidth]{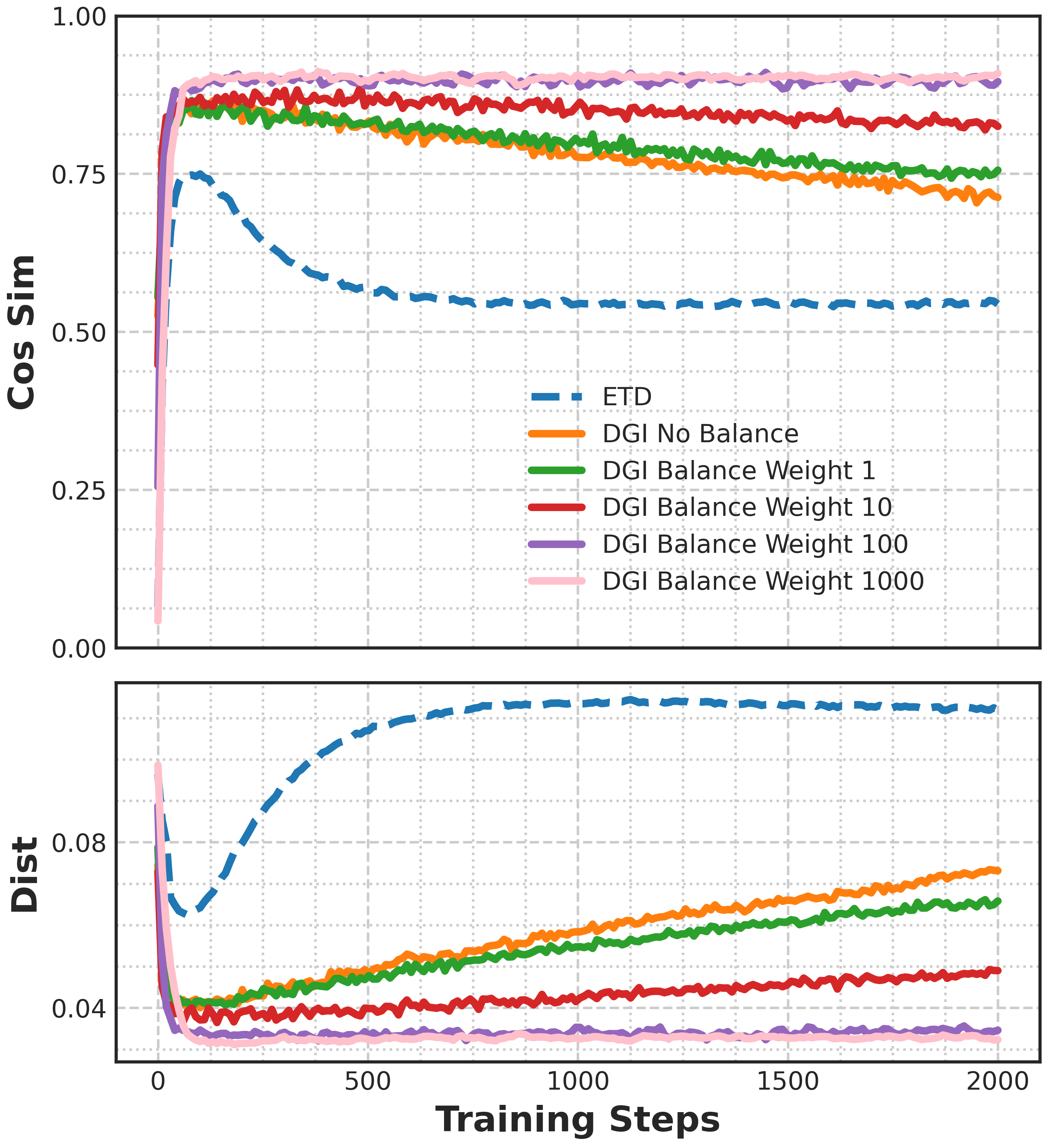}
    \label{fig:sub1}
\end{subfigure}
\hfill
\begin{subfigure}[b]{0.32\textwidth}
    \includegraphics[width=\textwidth]
    {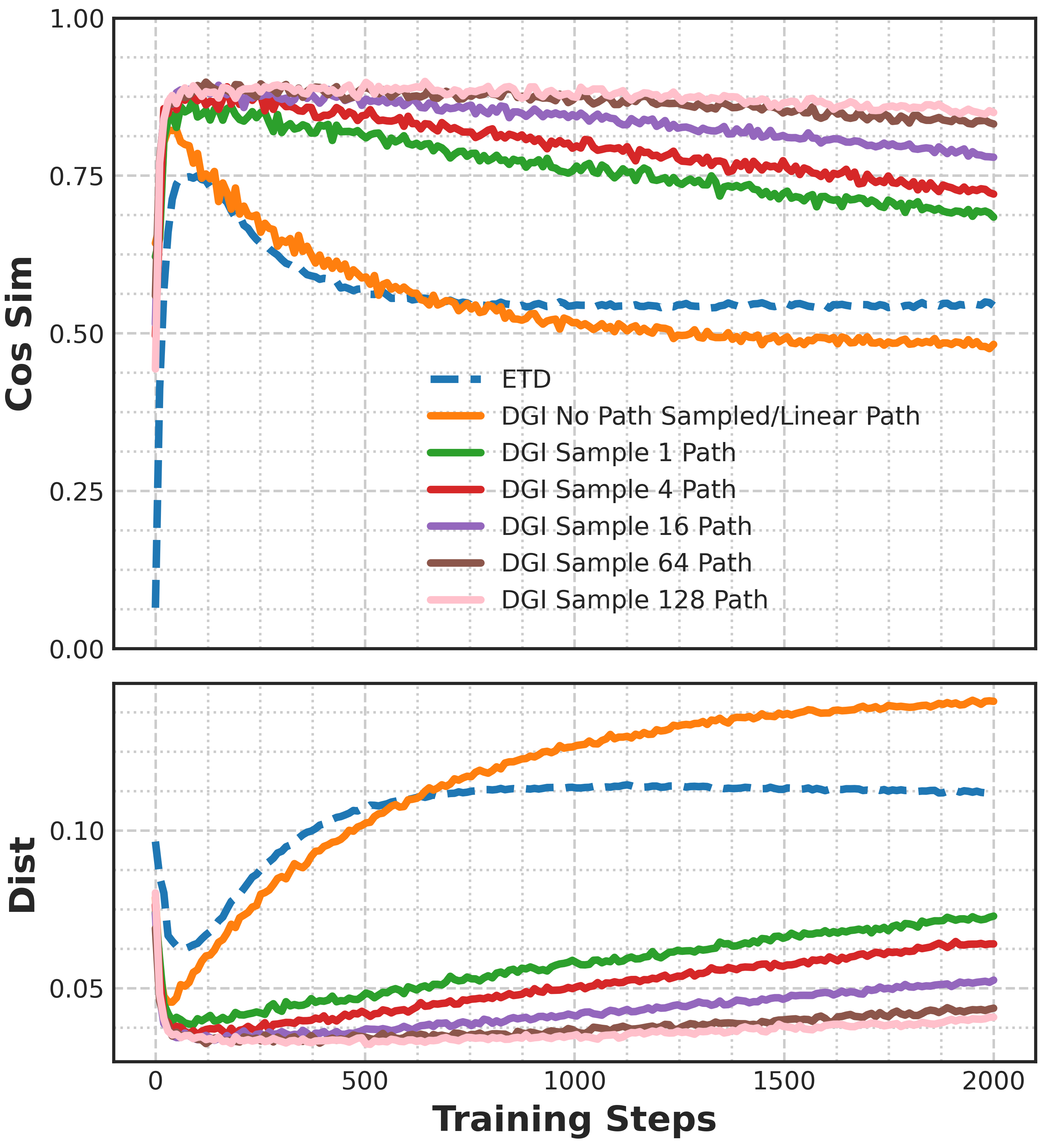}
    \label{fig:sub1}
\end{subfigure}
\hfill
\begin{subfigure}[b]{0.3275\textwidth}
    \includegraphics[width=\textwidth]{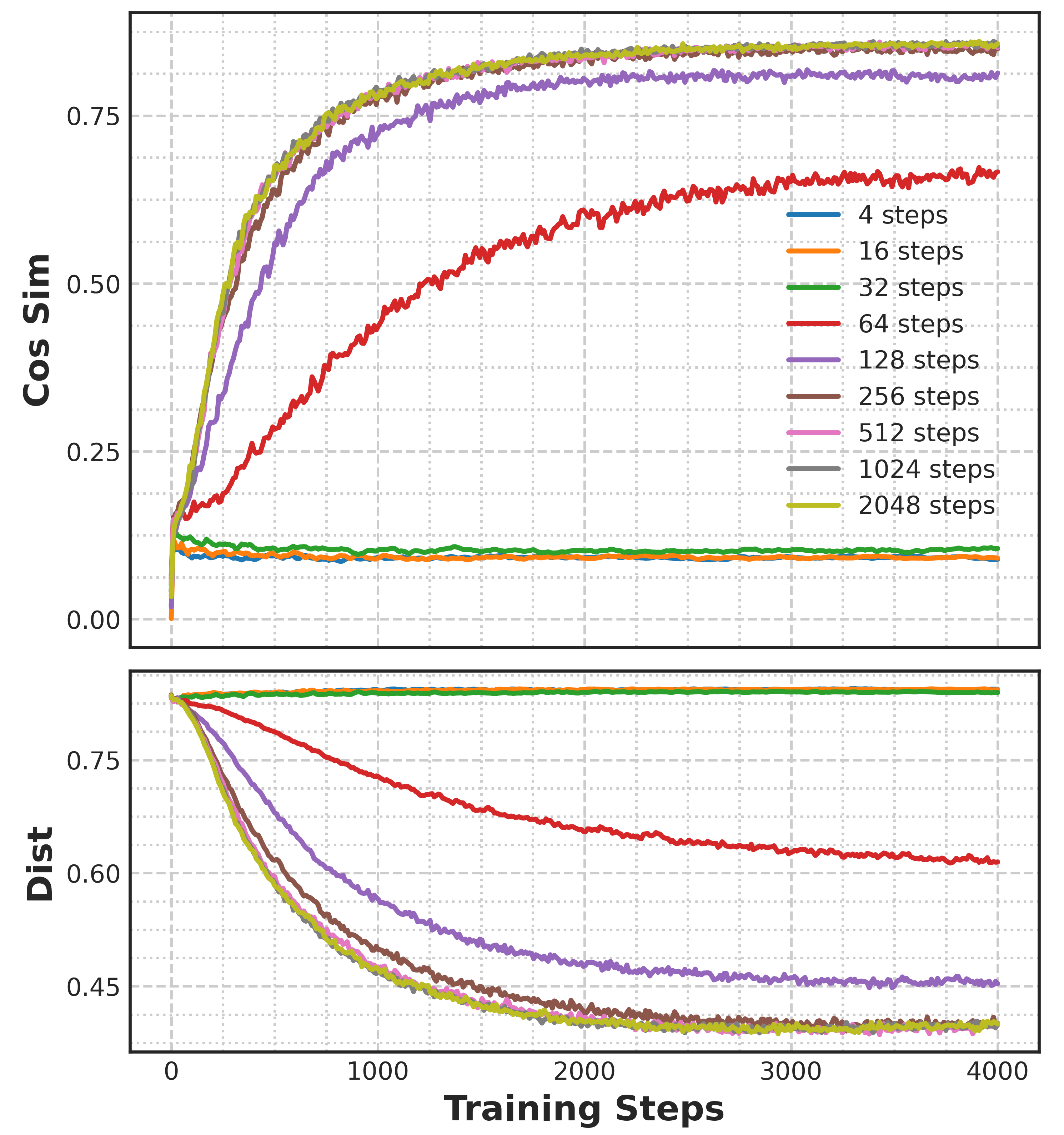}
    \label{fig:sub1}
\end{subfigure}
\caption{Ablation studies on the effects of the \textbf{\emph{balance weight}} (left), \textbf{\emph{number of sampled paths}} (middle), and \textbf{\emph{integration accuracy}} \emph{measured by the integration steps} (right) on the performance of gradient estimation.}
\label{fig:ablation}
\end{figure*}
\subsection{Optimization Performance}\label{sec:exp-perf}
We investigate the effectiveness of our proposed approach in optimizing the (unknown) value function $\nu(\theta)$. We compare the following methods: \textbf{\emph{Oracle in the Dataset (OC)}}: We report the value of the best design parameter in the training dataset. (Note that this value \emph{\textbf{cannot be achieved in practice}} because the true value $\nu(\theta)$ cannot be observed); \textbf{\emph{Random Search (RS)}} randomly samples $M$ design parameters $\theta$; \textbf{\emph{ETD}} and \textbf{\emph{DGI}}: We consider two variations to initialize gradient descent: to randomly select a $\theta$ in the offline dataset (denoted by \textbf{\emph{R}}) and to greedily choose $\theta$ in the dataset with the smallest observed value of $y$ (denoted by \textbf{\emph{G}}).

We give each method $M=128$ attempts and report the average ground true value. We choose this metric because the true value of SOBBO cannot be observed in practice, so it is crucial that an approach generates consistent good designs. We use gradient descent with the true black-box objective function to estimate an optimized value and use it to normalize the performance metrics of all methods through $(y-x)/|x|$ where $x$ is the estimated optimized value and $y$ is the performance metric of algorithms. The results are presented in Table~\ref{tab:scarce-main} and~\ref{tab:large-main}. It can be observed that the proposed methods significantly outperform the baseline of random search. Moreover, it even outperforms the (unobtainable) dataset oracle in most scenarios.
% For each method except OC, we have $M=128$ attempts. We report the $100^{\text{th}}$ percentile ground truth values (i.e., the best value) on this set of designs. We also report the $50^{\text{th}}$ percentile.

% \begin{table}[h]
% \caption{Performance for Large-Data Regime} \label{sample-table}
% \begin{center}
% \begin{tabular}{llllll}
% \textbf{}  &\textbf{Quadratic} &\textbf{Griewank} &\textbf{Ackley} \\
% \hline \\
% RS(ood)        & 0.0815&0.511 &0.761\\
% RS             & 0.0738&0.511 &0.746\\
% OC           &0.00071&0.500&0.380\\
% True sgd (RS)&0&0.500&0.541\\
% True sgd (greedy) &0&0.493&0.380\\
% ETD(RS)            &0.0208 & 0.507&0.558\\
% ETD(greedy)             & 0.0191& 0.507&0.382\\
% DGI(RS)& 0.0032&0.499 &0.446\\
% DGI(greedy) &0.0027 &0.499&0.375
% \end{tabular}
% \end{center}
% \end{table}

% \begin{table}[h]
% \caption{Performance for Large-Data Regime} \label{sample-table}
% \begin{center}
% \begin{tabular}{lll}
% \textbf{}  &\textbf{Beam}&\textbf{Vessel}  \\
% \hline \\
% RS(ood)        &0.215&0.185\\
% RS             &0.219&0.193\\
% OC           &0.003&0.0001\\
% True sgd (RS)&-0.0005& 0\\
% True sgd (greedy) &-0.0005&0\\
% ETD(RS)      &0.038&0.0275\\
% ETD(greedy)  &0.042&0.0264\\
% DGI(RS)&-0.0005&0.0172\\
% DGI(greedy)&-0.0005&0.0169
% \end{tabular}
% \end{center}
% \end{table}

\subsection{Gradient Estimation}\label{sec:exp-grad}
Given the observed significant improvements over baselines, we conduct an in-depth evaluation of \textbf{\emph{how effectively ETD and DGI estimate the gradient}} of black-box objective functions, which is their primary task. Moreover, we investigate the effectiveness of DGI's individual components.
% Since in real-world scenarios the true gradient is not accessible for evaluation, \textbf{\emph{competitive methods must be stable during training}}, as well as robust to noise. 
% It is crucial to note that the true gradient is not known in real-world scenarios, making it \textbf{\emph{difficult to reliably assess true performance}}. 

In particular, since in real-world scenarios the true gradient is not accessible for evaluation, there is no clear signal to indicate when to halt the training process. The convergence of training loss does not necessarily correspond to optimal gradient estimation, particularly in the scarce-data regime. Instead, it may indicate overfitting to noise, resulting in poor gradient estimation. (See Appendix~\ref{sec:exp-overfit} for a comparison of gradient performance and training loss). Therefore, \textbf{\emph{competitive methods must maintain accurate gradient estimates throughout the training process}}.

{\bf Metrics.} We use two metrics to evaluate the gradient estimation performance. Let $\psi$ be the true gradient and $\widehat{\psi}$ the estimation,
\begin{itemize}
\setlength\itemsep{-0.2em}
    \item Cosine Similarity: \textbf{Cos Sim}$\left(\psi, \widehat{\psi}\right) = \frac{\|\psi - \widehat{\psi}\|_2^2}{\|\psi\|_2\|\widehat{\psi}\|_2}$;
    \item Norm Distance: \textbf{Dist}$\left(\psi, \widehat{\psi}\right) = \|\psi-\widehat{\psi}\|_2$.
\end{itemize}
These two metrics together evaluate the accuracy of both the \textbf{\emph{direction}} and \textbf{\emph{magnitude}} of the estimated gradients.  
% We then compare the estimated true gradient with the estimated gradient of our methods, using the same $\theta$ and $x's$. This is done by comparing both the direction and magnitude, through the calculation of cosine similarity and the norm of the gradient distance. 
The evaluation is averaged over $1000$ randomly sampled designs (following the same data distribution of $\theta$) to have a comprehensive evaluation.

{\bf Baselines.} To assess the effectiveness of the proposed training objectives for DGI, we compare ETD with five variations of DGI: (i) \textbf{\emph{DGI-full}} is the complete DGI with balancing loss and path independence loss optimized with 64 sampled paths; (ii) \textbf{\emph{DGI-naive}} is without balancing loss and path independence loss (trained with only estimation loss on the linear path integral); (iii) \textbf{\emph{DGI-path1}} and (iv) \textbf{\emph{DGI-path64}} are trained with only path independence loss optimized with $1$ and $64$ sampled path integral, respectively; and (v) \textbf{\emph{DGI-path1/bal}} is trained with balancing loss and path independence loss with $1$ sampled path.

{\bf Results.} The experimental results are presented in Figure~\ref{fig:sim-scarce} and~\ref{fig:sim-large1}. In scarce-data regimes, we observe that ETD demonstrates transient gradient estimation quality across all tested optimization problems: its performance quickly degrades during the training process, providing little opportunity to determine the correct stopping time. In contrast, 
DGI with balancing loss or path sampling converges to superior gradient estimation and maintains as training progresses. Notably, DGI without these two components demonstrates similar behavior to that of ETD, demonstrating the importance of these components. 
% See Section~\ref{sec:exp-ablation} for more ablation studies about the effectiveness of these components.
% Our experiments demonstrate that DGI significantly mitigates the issue of transient gradient estimation quality seen in ETD and converges to superior gradient estimates across all experiments in both small and large-data regimes. In figure, we compare the performance of ETD and five different configurations of DGI, i.e., EGX with or without balancing, using either linear, single, or multiple paths sampling strategy. In small data regime, DGI with balancing or multiple-path sampling converges to superior gradient estimation and maintains stability as training progresses. On the other hand, ETD and DGI without employing these two strategies are particularly sensitive to noise and quickly degrade to inferior convergence. (See ablation section for more on the effectiveness of the two strategies.) 
On the other hand, the results in Figure~\ref{fig:sim-large1} indicate that the discrepancy between DGI-full and other methods is reduced to a certain degree in the large-data regime. 

\subsection{Ablation Studies}\label{sec:exp-ablation}
We investigate the impact of three key hyperparameters on the performance of DGI: \textbf{\emph{(i) weight for balancing loss}}, \textbf{\emph{(ii) number of sampled paths}} for the path independence loss, and \textbf{\emph{(iii) integration accuracy}}. Additionally, we find that \textbf{\emph{DGI is considerably more robust to noise}}. 
% This robustness makes DGI more suitable for real-world applications. 
Due to space constraint, we defer this result and ablation setting to Appendix~\ref{sec:exp-more}. 
% All ablation experiments are conducted on the small NN dataset within the small data regime, employing an extended training process, except the ablation on integration accuracy, which is done on Trid dataset for a better demonstration.

The results are presented in Figure~\ref{fig:ablation}. Specifically, to evaluate the independent effect of the balance loss, we fix the number of sampled paths to $1$ and vary the weight for balance loss across values in $\{0, 1, 10, 100, 1000\}$. The results indicate that increasing the balance weight significantly improves gradient estimation accuracy and stabilizes performance. To isolate the effect of path sampling, we set the balance weight to zero and vary the number of sampled paths in $\{0\text{(only linear path)}, 1, 4, 16, 64, 128\}$. Similar to the effects of increasing the balance weights, increasing the number of sampled paths improves gradient estimation and performance stability. 
% It is conceivable that combining these two methods could result in better performance on more complex tasks. 
As a critical component of the DGI algorithm, integration is a typical computational bottleneck. 
To this end, our experiments show that \textbf{\emph{DGI achieves strong performance without requiring high integration accuracy}}: It performs well with only a moderate number of integration steps, efficiently preserving gradient information.

\section{Conclusion}\label{sec:dis}
% We have presented two approaches for optimizing black-box objective functions using historical data. ETD is easy to implement and theoretically guaranteed for large-data regimes. While DGI incurs higher computational costs, it offers superior gradient estimation, stable convergence, and robustness to noise. However, computational bottlenecks are typically not the primary concern in scarce-data regimes. An interesting research question is identifying the transition point from scarce-data to large-data regimes—determining the amount of data needed for ETD to perform sufficiently well, thus avoiding the potentially high computational cost of DGI in large-data scenarios. We plan to pursue this investigation in forthcoming work. 
We presented two approaches for optimizing black-box objective functions involving uncertainties using historical data. ETD is simple to implement and theoretically sound for large-data regimes, while DGI, though more computationally expensive, provides better gradient estimation, stable convergence, and noise robustness. 
The trade-off between performance and computational cost in the large-data regime is an interesting question and we will explore this in future work.

\section*{Acknowledgment}
The work of Hamid Jafarkhani, Vahid Tarokh and Ali Pezeshki was respectively supported in part by the collaborative NSF Awards CNS-2229467, CNS-2229468, and CNS-2229469.

\newpage
\appendix
\onecolumn
\hrule height4pt
\vskip .25in
\begin{center}
    {\Large\bfseries Supplementary Materials}
\end{center}
\vskip .20in
\hrule height1pt
\vskip .40in
% \aistatstitle{Supplementary Materials}
\section{Proofs of Theoretical Results}\label{sec:app-proof}

\textbf{Theorem~\ref{thm:consistency-of-gradient}}
\textit{
Let $\cG$ be a function class of bounded continuously differentiable functions such that $g \in \cG$. Assume that $\Theta$ and $\cX$ are compact. Let $\widehat g_n \in \cG$ be the solution to the Empirical Risk Minimization problem defined in~\eqref{eqn:obj-erm}, where $n$ is the total number of training samples in the offline dataset. Let 
$$
\eta(\theta) = \nabla_{\theta}\bbE_X[g(\theta,X)]\;\;\text{and}\;\;\widehat{\eta}_n(\theta) = \frac{1}{n}\sum^n_{i=1}\nabla_{\theta}\widehat g_n(\theta,X_i) 
$$
respectively denote the true gradient and the gradient estimate using $\widehat g_n$ and the offline dataset $\mathcal{D}$, then 
\begin{align*}\bbE_\theta\bigg[\big\|\eta(\theta)-\widehat{\eta}_n(\theta)\big\|\bigg]  \stackrel{P}{\rightarrow} 0,
\end{align*}
where $\stackrel{P}{\rightarrow}$ denotes convergence in probability. 
% Let $\cF$ be a function class of bounded continuous functions with bounded gradients. Assume that $g \in \cF$. Let $\widehat g_n \in \cF$ be the solution to the Empirical Risk Minimization problem as defined in Eq.~\eqref{eqn:obj-erm} where $n$ is the number of training samples. Let 
% $$
% \nu_n(\theta) = \bbE_X[\nabla_{\theta}\widehat g_n(\theta,X)]\;\;\text{and}\;\; \nu(\theta) = \bbE_X[\nabla_{\theta}g(\theta,X)]
% $$
% denote the gradient estimate using $\widehat g_n$ and the gradient of the true loss function, then 
% \begin{align*}\bbE_\theta\bigg[\big\|\nabla_{\theta}\nu_n(\theta)-\nabla_{\theta}\nu(\theta)\big\|\bigg]  \stackrel{P}{\rightarrow} 0,
% \end{align*}
% where $\stackrel{P}{\rightarrow}$ denotes convergence in probability with $X$ and $\theta$ following the data distribution described in Section~\ref{sec:problem_setting}. 
}

\begin{proof}[Proof of Theorem~\ref{thm:consistency-of-gradient}]
In this proof, we will use $f$ to denote specific functions in $\cG$ and $g$ to denote the true objective function. Consider the loss function 
$$
\cR(f) = \bbE_{\theta,X,Y}[(f(\theta,X)-Y)^2],
$$
where the joint distribution of $(\theta,X,Y)$ follows the data generating process described in Section~\ref{sec:intro}. Note that the optimum $f^*$ is obtained at 
$$
f^*(\theta,x) = \bbE[Y|\theta,X=x] = \bbE[g(\theta,x)+\epsilon|\theta,X=x] = g(\theta,x).
$$
Moreover, $\cR$ can be decomposed as
\begin{align*}
    \cR(f) &= \bbE_{\theta,X,Y}[(f(\theta,X)-Y)^2] = \bbE_{\theta,X,\epsilon}[(f(\theta,X)-g(\theta,X)+\epsilon)^2] \\
    & = \bbE_{\theta,X}[(f(\theta,X)-g(\theta,X))^2] + 2\bbE_{\theta,X}[f(\theta,X)-g(\theta,X)]\bbE[\epsilon] + \bbE[\epsilon^2] \\
    & = \bbE_{\theta,X}[(f(\theta,X)-g(\theta,X))^2] + \sigma^2,
\end{align*}
where the last equality is due to $\bbE[\epsilon]=0$ and independence of $\epsilon$ and $\bbE_{\theta,X}[f(\theta,X)-g(\theta,X)]$. Hence, $\inf_{f \in \cF}\cR(f) = \cR(g) = \sigma^2$. By classical learning theory results~\citep{shorack2009empirical, vapnik1991principles}, if $g \in \cG$ and $\cG$ is a Glivenko–Cantelli (GC) function class (which includes the family of bounded continuous functions), we have
\begin{equation}\label{eqn:g-convergence}
    \cR(\widehat g_n) - \inf_{f \in \cF}\cR(f) \rightarrow 0 \Rightarrow \bbE_{\theta,X}[(\widehat g_n(\theta,X)-g(\theta,X))^2] \rightarrow 0.
\end{equation}
Let $\eta_n(\theta) = \bbE_{X}[\nabla_{\theta}\widehat g_n(\theta,X)]$. By triangle inequality of norm, we have
\begin{align}\label{eqn:grad-error-triangle}
    \bbE_\theta\bigg[\big\|\eta(\theta)-\widehat{\eta}_n(\theta)\big\|\bigg] \le \bbE_\theta\bigg[\big\|\eta(\theta)-\eta_n(\theta)\big\|\bigg] +  \bbE_\theta\bigg[\big\|\eta_n(\theta)-\widehat{\eta}_n(\theta)\big\|\bigg].
\end{align}
We first consider the second term in the upper bound. By Jansen's inequality, we have 
\begin{align*}
    \bbE_\theta\bigg[\big\|\eta_n(\theta)-\widehat{\eta}_n(\theta)\big\|\bigg] \le \bbE_{\theta}\left[\|\bbE_X[\nabla \widehat{g}_n(\theta,X)] -\bbE_n[\nabla \widehat{g}_n(\theta,X)]\|\right],
\end{align*}
where $\bbE_n$ denotes expectation with respect to the empirical distribution. Now consider the following function class 
$$
\cG_{\Theta}=\{f_{g,\theta}(x) \mapsto \nabla_\theta g(\theta,x) | \theta\in\Theta, g \in \cG\}.
$$
Because $\cG$ is continuously differentiable and $\cX$ is compact, we have that $\cG_\Theta$ is also a Glivenko–Cantelli (GC) function class. Thus, it holds that $\sup_{f \in \cG_\Theta}\|\bbE_X[f(X)]-\bbE_n[f(X)]\|$ converges to $0$ almost surely, which implies that as $n \rightarrow \infty$,
$$
\bbE_{\theta}\left[\|\bbE_X[\nabla \widehat{g}_n(\theta,X)] -\bbE_n[\nabla \widehat{g}_n(\theta,X)]\|\right] \rightarrow 0,
$$
almost surely. 

Now we consider the first term in the upper bound in~\eqref{eqn:grad-error-triangle}. As $L_2$ convergence implies $L_1$ convergence and with the result in~\eqref{eqn:g-convergence}, we also have $\bbE_{\theta,X}[\|\widehat g_n(\theta,X)-g(\theta,X)\|] \rightarrow 0$. Recall that 
$$
\nu_n(\theta) = \bbE_X[\widehat g_n(\theta,X)]; \quad \nu(\theta) = \bbE_X[g(\theta,X)].
$$
Thus, 
\begin{equation}
     \bbE_\theta[\|\nu_n(\theta)-\nu(\theta)\|] =  \bbE_\theta[\|\bbE_X[\widehat g_n(\theta,X)]-\bbE_X[g(\theta,X)]] \le \bbE_{\theta,X}[\|\widehat g_n(\theta,X)-g(\theta,X)\|] \rightarrow 0.
\end{equation}
As both $\nu_n$ and $\nu$ are continuous because $g,\widehat{g}_n \in \cF$ are continuous, this implies that $\nu_n(\theta)$ converges to $\nu(\theta)$ almost surely. For any $j \in [d_\theta]$, let $e_j$ denote the zero vector in $\bbR^{d_\theta}$ with only the $j$-th element as $1$: 
\begin{align*}
    \lim_{n \rightarrow +\infty}\bbE_\theta\bigg[\big\|\frac{\partial}{\partial\theta_j}(\nu_n(\theta)-\nu(\theta)\big\|\bigg] &=  \bbE_\theta\bigg[\lim_{n \rightarrow +\infty}\big\|\lim_{h \rightarrow 0_+}(\nu_n(\theta+h e_j)-\nu(\theta+h e_j)-(\nu_n(\theta)-\nu(\theta)))/h\big\|\bigg] \\
    &\le \bbE_\theta\bigg[\lim_{n \rightarrow +\infty}\lim_{h \rightarrow 0_+}\big\|\nu_n(\theta+h e_j)-\nu(\theta+h e_j)-(\nu_n(\theta)-\nu(\theta))\big\|/h\bigg]. %\\
\end{align*}
Now, consider the function 
$$
\beta^h_n(\theta) = \left(\nu_n(\theta+h e_j)-\nu(\theta+h e_j)-(\nu_n(\theta)-\nu(\theta))\right)/h.
$$
For all $n$ and $\theta$, the limit
$
\lim_{h \rightarrow 0_+}\|\beta^h_n(\theta)\|
$
exists as 
\begin{align*}
    \lim_{h \rightarrow 0_+} \beta^h_n(\theta) = \frac{\partial}{\partial\theta_j}(\nu_n(\theta)-\nu(\theta)) := c(\theta),
\end{align*}
and hence
\begin{align*}
\lim_{h \rightarrow 0_+} \bigg\|\bigg(\|\beta^h_n(\theta)\| - \|c(\theta)\|\bigg)\bigg\| = \lim_{h \rightarrow 0_+}\|\beta^h_n(\theta)-c(\theta)\| = 0,
\end{align*}
where the last equality comes from the fact that $c$ is the limit of $\beta^h_n(\theta)$. This implies that $\|\beta^h_n(\theta)\|$ converges to $\|c\|$. Moreover, by the assumption that $\cF$ includes functions with bounded gradients (let there be an upper bound denoted by $L$), we have
$$
\|\nabla_\theta\nu(\theta)\| = \|\bbE_{X,\theta}[\nabla_\theta g(X,\theta)]\| \le \bbE_{X,\theta}[\|\nabla_\theta g(X,\theta)\|] \le L.
$$
As $\widehat g_n \in \cF$, by the same argument,  $\|\nabla_\theta\nu_n(\theta)\| \le L$. Thus, the function $\nu_n(\theta)-\nu(\theta)$ is $2L$-bounded. Then, by the mean-value theorem, $\|\beta^h_n(\theta)\| \le 2Lh/h = 2L$.
Hence, by the dominated convergence theorem, we have
$$
\bbE_\theta\bigg[\big\|\frac{\partial}{\partial\theta_j}(\nu_n(\theta)-\nu(\theta)\big\|\bigg] \le \bbE_\theta\bigg[\lim_{h \rightarrow 0_+}\|\beta^h_n(\theta)\|\bigg] = \lim_{h \rightarrow 0_+}\bbE_\theta\bigg[\|\beta^h_n(\theta)\|\bigg].
$$
In particular, for any fixed $n$, we have
$$
\lim_{h \rightarrow 0_+}\bbE_\theta\bigg[\|\beta^h_n(\theta)\|\bigg] = \bbE_\theta\bigg[\lim_{h \rightarrow 0_+}\|\beta^h_n(\theta)\|\bigg] = \bbE_\theta[\|c(\theta)\|].
$$
As $\|c(\theta)\|$ is bounded, we have $ \bbE_\theta[\|c(\theta)\|]$ well defined and thus $\lim_{h \rightarrow 0_+}\bbE_\theta\bigg[\|\beta^h_n(\theta)\|\bigg]$ exists for any fixed $n$. For any fixed $h$, by triangle inequality of the norm, we have almost surely
$$
\lim_{n \rightarrow +\infty}\|\beta^h_n(\theta)\|  \le \left(\lim_{n \rightarrow +\infty}\big\|\nu_n(\theta+h e_j)-\nu(\theta+h e_j)\big\|+\lim_{n \rightarrow +\infty}\big\|\nu_n(\theta)-\nu(\theta)\big\|\right)/h = 0,
$$
because $\nu_n(\theta)$ converges to $\nu(\theta)$ almost surely. As we have already showed that $\|\beta^h_n(\theta)\| \le 2L$, again by dominated convergence theorem, we have
$$
\lim_{n \rightarrow +\infty}\bbE_\theta[\|\beta^h_n(\theta)\|] = \bbE_{\theta}[\lim_{n \rightarrow +\infty}\|\beta^h_n(\theta)\|] = 0.
$$
Combining the above results and Moore-Osgood theorem, we derive
\begin{align*}
    \lim_{n \rightarrow +\infty}\bbE_\theta\bigg[\big\|\frac{\partial}{\partial\theta_j}(\nu_n(\theta)-\nu(\theta)\big\|\bigg] \le \lim_{n \rightarrow +\infty}\lim_{h \rightarrow 0_+}\bbE_\theta\bigg[\|\beta^h_n(\theta)\|\bigg] = \lim_{h \rightarrow 0_+}\lim_{n \rightarrow +\infty}\bbE_\theta\bigg[\|\beta^h_n(\theta)\|\bigg] = \lim_{h \rightarrow 0_+}0 = 0.
\end{align*}
In terms of the gradient's convergence, using the fact that for any norm $\|\cdot\|$ and vector $x \in \bbR^{d_\theta}$, $\|x\|\le \sum_{j=1}^{d_\theta}\|x_j\|$ where $x_j$ is the $j$-th element of $x$, we have
\begin{align*}
    &\lim_{n \rightarrow +\infty}\bbE_\theta\bigg[\big\|\nabla_\theta(\nu_n(\theta)-\nu(\theta))\big\|\bigg] \le    \lim_{n \rightarrow +\infty}\sum_{j=1}^{d_\theta}\bbE_\theta\bigg[\big\|\frac{\partial}{\partial\theta_j}(\nu_n(\theta)-\nu(\theta))\big\|\bigg] \\
    &\le \sum_{j=1}^{d_\theta}\lim_{n \rightarrow +\infty}\bbE_\theta\bigg[\big\|\frac{\partial}{\partial\theta_j}(\nu_n(\theta)-\nu(\theta)\big\|\bigg]=0.
\end{align*}
This concludes the proof.
\end{proof}

\textbf{Lemma~\ref{lemma:validity-of-recon-loss}}
\textit{
Consider the data generation process $Y = g(\theta,X) + \epsilon$ where $\epsilon$ is independent of $(\theta,X)$ and $\bbE[\epsilon]=0,\mathrm{Var}[\epsilon^2]=\sigma^2$. 
% Let $L_r$ be defined as in~\eqref{eqn:recon_loss_data}. 
Then we have $\min_{\phi'}L_r(\phi')=2\sigma^2$. Moreover, if $L_r(\phi)=2\sigma^2$, then for any two random pairs of $(\theta_1,x_1), (\theta_2,x_2) \in \Theta\times\cX$ following the data distribution, 
\begin{align*}
   g(\theta_1,x_1)-g(\theta_2,x_2) = \int_{\theta_2,x_2}^{\theta_1,x_1}\nabla_{\theta,x}h_\phi(\Tilde{\theta},\Tilde x)d\Tilde{\theta} d\Tilde x
\end{align*}
almost surely.
% Consider the data generation process that $Y = g(\theta,X) + \epsilon$ where $\epsilon$ is independent of $(\theta,X)$ and $\bbE[\epsilon]=0,\bbE[\epsilon^2]=\sigma^2$. Let $L_r$ be defined as in Eq.~\eqref{eqn:recon_loss_data}. Then, $\min_{\phi'}L_r(\phi')=2\sigma^2$. Moreover, if $L_r(\phi)=2\sigma^2$, then it holds almost surely over $\Theta\times\cX$ that for any two pairs of $(\theta_1,x_1), (\theta_2,x_2) \in \Theta\times\cX$, 
% \begin{align*}
%    g(\theta_1,x_1)-g(\theta_2,x_2) = \int_{\theta_2,x_2}^{\theta_1,x_1}\nabla_{\theta,x}h_\phi(\Tilde{\theta},\Tilde x)d\Tilde{\theta}d\Tilde x.
% \end{align*}
}

\begin{proof}[Proof of Lemma~\ref{lemma:validity-of-recon-loss}]
\begin{align*}
     L_r(\phi) &= \bbE_{(\theta_1,X_1,Y_1),(\theta_2,X_2,Y_2)}\bigg[\left(Y_1-Y_2-\int^{\theta_2,X_2}_{\theta_1,X_1}h_\phi(\Tilde{\theta},\Tilde x)d\Tilde{\theta}d\Tilde x\right)^2\bigg] \\
            &= \bbE_{(\theta_1,X_1,\epsilon_1),(\theta_2,X_2,\epsilon_2)}\bigg[\left(g(\theta_1,X_1)+\epsilon_1-g(\theta_2,X_2)-\epsilon_2-\int^{\theta_2,X_2}_{\theta_1,X_1}h_\phi(\Tilde{\theta},\Tilde x)d\Tilde{\theta}d\Tilde x\right)^2\bigg]\\
            & = \bbE_{(\theta_1,X_1),(\theta_2,X_2)}\bigg[\left(g(\theta_1,X_1)-g(\theta_2,X_2)-\int^{\theta_2,X_2}_{\theta_1,X_1}h_\phi(\Tilde{\theta},\Tilde x)d\Tilde{\theta}d\Tilde x\right)^2\bigg] + \bbE_{\epsilon_1,\epsilon_2}[(\epsilon_1-\epsilon_2)^2] \\ 
            &+ \bbE_{(\theta_1,X_1,\epsilon_1),(\theta_2,X_2,\epsilon_2)}\bigg[\left(g(\theta_1,X_1)+\epsilon_1-g(\theta_2,X_2)-\epsilon_2-\int^{\theta_2,X_2}_{\theta_1,X_1}h_\phi(\Tilde{\theta},\Tilde x)d\Tilde{\theta}d\Tilde x\right)(\epsilon_1-\epsilon_2)\bigg].
\end{align*}
Note that by independence assumption of $\epsilon_1$ and $\epsilon_2$, we derive
$$
\bbE_{\epsilon_1,\epsilon_2}[(\epsilon_1-\epsilon_2)^2] = \bbE_{\epsilon_1,\epsilon_2}[\epsilon_1^2-2\epsilon_1\epsilon_2+\epsilon^2_2] = \bbE_{\epsilon_1}[\epsilon^2_1]+\bbE_{\epsilon_2}[\epsilon^2_2] -\bbE_{\epsilon_1}[\epsilon_1]\bbE_{\epsilon_2}[\epsilon_2] = 2\sigma^2.
$$
%Moreover, by the same reason,
Similarly, 
\begin{align*}
    &\bbE_{(\theta_1,X_1,\epsilon_1),(\theta_2,X_2,\epsilon_2)}\bigg[\left(g(\theta_1,X_1)+\epsilon_1-g(\theta_2,X_2)-\epsilon_2-\int^{\theta_2,X_2}_{\theta_1,X_1}h_\phi(\Tilde{\theta},\Tilde x)d\Tilde{\theta}d\Tilde x\right)(\epsilon_1-\epsilon_2)\bigg] = \\
    & \bbE_{(\theta_1,X_1),(\theta_2,X_2)}\bigg[\left(g(\theta_1,X_1)+\epsilon_1-g(\theta_2,X_2)-\epsilon_2-\int^{\theta_2,X_2}_{\theta_1,X_1}h_\phi(\Tilde{\theta},\Tilde x)d\Tilde{\theta}d\Tilde x\right)\bigg] \bbE_{\epsilon_1,\epsilon_2}[\epsilon_1-\epsilon_2] = 0.
\end{align*}
Hence, we have
\begin{equation}\label{eqn:recon-loss-decom}
     L_r(\phi) = \bbE_{(\theta_1,X_1),(\theta_2,X_2)}\bigg[\left(g(\theta_1,X_1)-g(\theta_2,X_2)-\int^{\theta_2,X_2}_{\theta_1,X_1}h_\phi(\Tilde{\theta},\Tilde x)d\Tilde{\theta}d\Tilde x\right)^2\bigg] + 2\sigma^2.
\end{equation}

Because the first term is non-negative, $\inf_{\phi'}L_r(\phi') \ge 2\sigma^2$. By choosing $h_\phi = \nabla g$, $L_r(\phi) = 2\sigma^2$. Hence, $\min_{\phi'}L_r(\phi')=2\sigma^2$. The almost sure equality follows from the fact that when the expectation of a non-negative function (first term on the right hand side of~\eqref{eqn:recon-loss-decom}) is zero, there only exists a set of measure zero such that the function value is not zero. 
\end{proof}

\section{Path  Parametrization and Sampling}\label{sec:path-param}
For simplicity of presentation, we present the path parametrization for $\theta$. The case for the concatenated $(x,\theta)$ pair is identical. 

For any given pair of $\theta, \theta' \in \bbR^d$, the path from $\theta$ to $\theta'$ can be parameterized by a function $r(t) = [r_1(t), r_2(t), \dots, r_d(t)]$ where  $r_i(t): \bbR \rightarrow \bbR$ for all $i \in [d]$ with the restriction that $r(0)=\theta$ and $r(1)=\theta'$. When $r_i(t)$ is chosen to be of polynomial type with highest degree $\tau$, i.e.,
$$
r_i(t) = a_{i0}+ a_{i1}t + \cdots + a_{i\tau}t^\tau.
$$
We have the equations 
\begin{align}
    &a_{i0} = \theta_{[i]} \quad \forall i \in \{1,\dots,d\}, \\
    &a_{i0}+ a_{i1} + \cdots + a_{i\tau} = \theta'_{[i]} \quad \forall i \in \{1,\dots,d\},
\end{align}
where $\theta_{[i]}$ is the $i$-th element of $\theta$. In other words, set $a_{i0} = \theta_{[i]}$ for all $i \in \{1,\dots,d\}$ and find $a_{i1}, a_{i2}, \dots, a_{i\tau}$ such that $$a_{i1} + \cdots + a_{i\tau} = \theta'_{[i]}-\theta_{[i]}$$ for all $i \in \{1,\dots,d\}$. In this work, we use a simple approach that first randomly samples $a_{i1}, \dots, a_{i(\tau-1)}$ and subsequently sets $a_{i\tau} = \theta'_{[i]}-\theta_{[i]}-\sum_{j=1}^{\tau-1}a_{ij}$.

\section{Details on Simulation Datasets and Real-World Application Datasets}\label{sec:data-gen}
\subsection{Simulation Objective Functions and Dataset Generation}
\begin{table}[H]
\centering
\caption{Simulation Objective functions}
\label{tab:sim-funcs}
\begin{tabular}{@{}ll@{}} % Alignment for each cell: left-left-right
\toprule
\textbf{Name} & \textbf{Formula}  \\ \midrule
Linear        & $g(x,\theta) =x^TS\theta\;\;$ where $S$ is the function parameter                     \\
Quadratic*        & $g(x,\theta) =(x^TS\theta)^2\;\;$ where $S$ is the function parameter                       \\
NN small        & Multi-layer perceptrons with input $\mathbf{z}$ and two  hidden layers of 100 neurons 
\\&and output dimension 1                 \\
NN large        & Multi-layer perceptrons with input $\mathbf{z}$  and two hidden layers of 1000 neurons
\\&and output dimension 1                   \\ 
Perm        & $g(\mathbf{z} ) = \sum_{i=1}^d \left( \sum_{j=1}^d (j + \beta) \left( z_{j}^i - \frac{1}{j^i} \right) \right)^2$ \\
Rosenbrock        & $g(\mathbf{z} ) = \sum_{i=1}^{d-1} \left( 100(z_{i+1} - z_i^2)^2 + (z_i - 1)^2 \right)
$                       \\
Zakharov        & $g(\mathbf{z} ) = \sum_{i=1}^d z_i^2 + \left( \sum_{i=1}^d 0.5iz_i \right)^2 + \left( \sum_{i=1}^d 0.5iz_i \right)^4$                        \\
Trid        & $g(\mathbf{z} ) = \sum_{i=1}^d (z_i - 1)^2 - \sum_{i=2}^d z_i z_{i-1}
$                    \\
Dixon-Price        & $g(\mathbf{z}) = (z_1 - 1)^2 + \sum_{i=2}^d i (2z_i^2 - z_{i-1})^2
$                     \\
Griewank*        & $g(\mathbf{z}) = \sum_{i=1}^d \frac{z_i^2}{4000} - \prod_{i=1}^d \cos\left(\frac{z_i}{\sqrt{i}}\right) + 1
$                        \\
Ackley*        & $g(\mathbf{z}) = -a \exp \left( -b\sqrt{  \frac{1}{d} \sum_{i=1}^d z_i^2 } \right) - \exp \left( \frac{1}{d} \sum_{i=1}^d \cos(c z_i) \right) + a + \exp(1)
$                     \\\bottomrule

\end{tabular}
\end{table}
For the Linear and Quadratic functions in Table~\ref{tab:sim-funcs}, \(x\) represents data and \(\theta\) is the control parameter we aim to optimize. Both \(x\) and \(\theta\) are 3-dimensional vectors. For the other functions in Table~\ref{tab:sim-funcs}, \(z\) is a 6-dimensional vector, with the first three components representing \(x\) and the last three components representing \(\theta\). To unify all functions on the same scale, we shift and scale each function such that the most significant features are preserved within the domain of the unit hypercube. This is achieved by applying constant multiplication and addition operations to $x$. Subsequently, a constant scaling factor is applied to the shifted function to ensure that the resulting objective values remain within a suitable range, thereby maintaining numerical stability. Without loss of generality, \(x\) and \(\theta\) are sampled uniformly from the unit hypercube.

When generating data, we add zero-mean Gaussian noise $\epsilon$ to the outputs of each function. To ensure that the added noises are proportional to the scales of the objective functions, we choose the variance of $\epsilon$ so that the \emph{single-sample Signal-to-Noise Ratio} (S3NR) is 0.5 for the scarce-data regime and 0.2 for large regime, a very challenging setting that better reflects real world scenarios. Specifically, S3NR for a given objective function $g$ is defined as 
\[\text{S3NR}(g, \epsilon) = \frac{\bbE[g^2]}{\bbE[\epsilon^2]}.\]
Note that the total \emph{effective} SNR increases with larger number of samples. While the minimum required total SNR for accurate gradient estimation is a valuable research question, we will investigate it in the work that follows. Functions marked with an asterisk $(*)$ in Table~\ref{tab:sim-funcs} are considered in the large-data regime, while the others are in the scarce-data regime.

\subsection{Settings of Real-World Applications and Dataset Generation}\label{sec:real-setting}
The datasets for the welded beam design problem and the pressure vessel design problem are generated using objective functions in a manner similar to those in simulation problems. The datasets for the other three problems are generated using the simulator provided by the SimOpt testbed~\citep{simoptgithub}.

\textbf{Welded Beam Design Problem}

The welded beam problem is a common engineering optimization challenge aimed at finding the optimal dimensions of a beam, such that the fabrication cost is minimized. In our setup, we consider the underlying objective function, where $z_1$, $z_2$, and $z_3$ are components of $\theta$, and $z_4$ represents $x$:
\[
g(\mathbf{z}) = 1.10471z_1^2z_2 + 0.04811z_3z_4(14 + z_2)
\]
Without loss of generality, we estimate the gradient of the cost function with respect to $\theta$ over the domain of the unit hypercube. The datasets are generated with $\theta$ and $x$ both uniformly sampled from the domain.

\textbf{Pressure Vessel Design Problem}

The pressure vessel design problem aims to find the optimal dimensions of a shell to minimize the total cost of material, forming, and welding. Similar to welded beam design problem, we consider the underlying objective function where $z_1$, $z_2$, and $z_3$ are components of $\theta$, and $z_4$ represents $x$:

\[
g(\mathbf{z}) = 0.6224 z_1 z_3 z_4 + 1.7781 z_2 z_3^2 + 3.1661 z_1^2 z_4 + 19.84 z_1^2 z_3
\]

Without loss of generality, we estimate the gradient of the cost function with respect to $\theta$ over the domain of the unit hypercube. The datasets are generated with $\theta$ and $x$ both uniformly sampled from the domain.

\textbf{Continuous Newsvendor Problem (CntNv)}

Continuous Newsvendor Problem is also known as the Single-Period Problem. In the problem, a vendor orders a fixed quantity of liquid ($\theta$) at the start of the day, incurring a per-unit cost ($x$). The product is sold at a constant price ($s$), and unsold liquid can be salvaged at a constant price ($w$) by day's end.

The daily random demand for the liquid follows a Burr Type XII distribution $F(z) = 1 - (1 + z^\alpha)^{-\beta}$ where $z$, $\alpha$, and $\beta$ are positive. The objective is to determine the optimal order quantity to maximize the expected profit. Following the recommended search scope, we generate datasets by uniformly sampling $x$ within the range of 2.5 to 7.5 and $\theta$ within the range of 0 to 1. Subsequently, we obtain the objective values using the simulator provided by the SimOpt testbed~\citep{simoptgithub}. In other words, this is a true black-box optimization problem without explicit formula for the objective function. The true gradient of $\theta$ is estimated by the SimOpt testbed. Since $\theta$ is one-dimensional in this scenario, we only present the results on the estimated gradient's norm distance in Figure~\ref{fig:real-sim}.

\textbf{M/M/1 Queue}

% \JD{What exactly is MM1 Queue optimizing for?}
An M/M/1 queue is a stochastic process with a state space representing the number of customers in the system. Arrival times and service times both follow exponential distributions with different distribution parameters ($x$ and $\theta$, respectively). By adjusting $x$ and $\theta$, the distributions of arrival times and service times can be altered accordingly. The objective is to find the optimal service time distribution parameter $\theta$ (service rate) that minimizes the expected average sojourn time plus a penalty for increasing the service rate $c\theta ^2$ where c is a positive constant. Following the recommended search scope, we generate datasets by uniformly sampling $x$ and $\theta$ within the range of 1 to 6. Subsequently, we obtain the objective values using the simulator provided by the SimOpt testbed~\citep{simoptgithub}. The true gradient of $\theta$ is estimated by the SimOpt testbed. 
Since $\theta$ is one-dimensional in this scenario, we only present the results on the estimated gradient's norm distance in Figure~\ref{fig:real-sim}.

\textbf{Stochastic Activity Network (SAN)}

Stochastic Activity Networks (SANs), also known as PERT networks, are used in planning large-scale projects. The SAN graph used in our study follows a specific structure shown in Figure~\ref{fig:san-graph} where each arc $i$ is associated with a task of random duration. Task durations are independent and exponentially distributed with mean $\mu_i$. The objective is to minimize the total time spent on the longest path from node $a$ to node $i$, plus the sum of the costs of each arc, $1/\mu_i$. We consider the first eight parameters, $\mu_1$ to $\mu_8$, as $\theta$ and the last five parameters, $\mu_9$ to $\mu_{13}$, as $x$. Following the recommended search scope, we generate datasets by uniformly sampling $x$ and $\theta$ within a hypercube of range from 0.1 to 5. Subsequently, we obtain the objective values using the simulator provided by the SimOpt testbed~\citep{simoptgithub}. The gradient of $\theta$ is then estimated.

\begin{figure}[H]
\centering
% Row 1
    \includegraphics[width=0.45\textwidth]
    {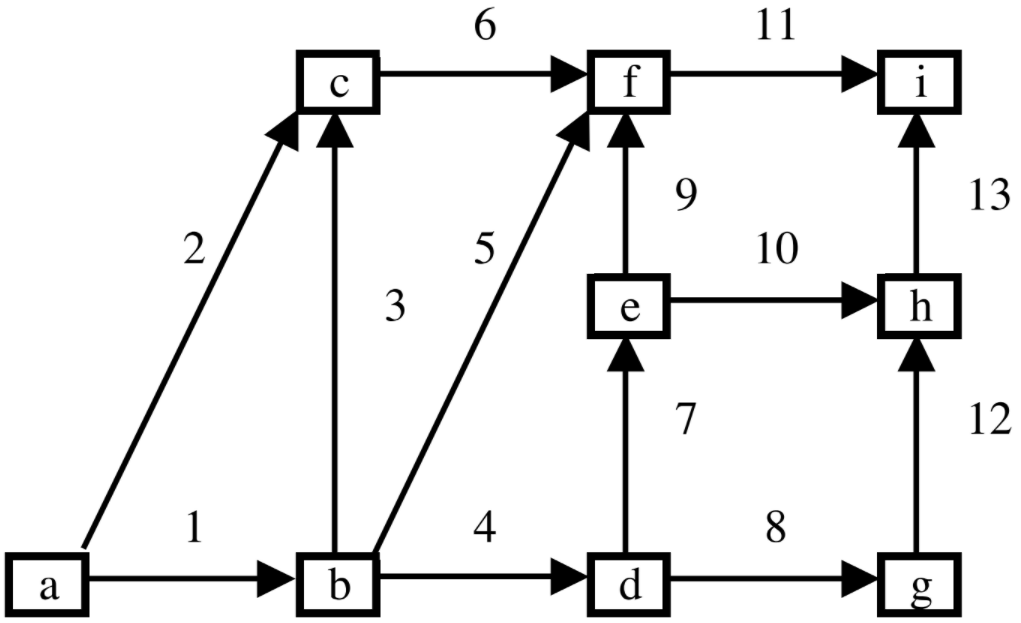}
    \caption{SAN project graph}
    \label{fig:san-graph}
\end{figure}

\subsection{Number of Samples}
\label{sec:num-samples}
\paragraph{Simulation Problems.} Datasets in the scarce-data regime consist of 128 samples. In the large-data regime, the Quadratic and Ackley functions consist of 20,000 samples each, and the Griewank function consists of 50,000 samples. 

\paragraph{Real-world Applications.} The datasets for the welded beam design problem, pressure vessel design problem, and continuous newsvendor problem are generated in the scarce-data regime, consisting of 128 samples each. The datasets for the M/M/1 queue problem and the stochastic activity network problem are generated in the large-data regime, consisting of 10,000 samples each. 

\section{Experimental Detail}\label{sec:exp-detail}

To ensure a fair comparison between proposed methods, we use the same structure as backbone for all experiments. Each method utilizes a multi-layer perceptron (MLP) with three hidden layers, each comprising $500$ neurons. For ETD, the input dimension of the backbone is the sum of the dimensions of \(x\) and \(\theta\), and the output dimension is one, corresponding to the estimated objective value. For all variations of DGI, the backbone's input and output dimensions are both the sum of the dimensions of \(x\) and \(\theta\), as these methods estimate the partial derivatives of the objective value with respect to each component of input. 

For all methods, the learning rate is fixed at $5e^{-4}$, and the batch size is set to $32$. For hyper-parameters specific to DGI methods, the integration step is set to 512 and the highest degree of sampled polynomial paths, $\tau$, is set to 10. Additionally, the balance weight and the number of sampled paths are configured according to each variation of DGI methods. We employ the Adam optimizer for training. The total number of training steps is determined to ensure convergence for each dataset, varying accordingly. The exact number of training steps is indicated in each experimental graph in Appendix~\ref{sec:app-full-results}.

During evaluation, we first compute the estimated true gradient and subsequently compare it with the gradients estimated by our methods. To obtain the estimated true gradient, we uniformly sample a set $X$, from the given dataset and uniformly sample a point $\theta$ from the predefined domain of $\theta$, e.g., the unit hypercube for all simulation problems. We then calculate the objective values for each pair of $x \in X$ and $\theta$, and perform back-propagation to derive the estimated true gradient for $\theta$, averaged over $X$. Subsequently, we use our methods to estimate the gradient of $\theta$ averaged over $X$. The comparison between our estimated gradients and the true gradients is then performed using cosine similarity and norm distance metrics, as detailed in Section~\ref{sec:exp}. We then average the results over 1,000 uniformly sampled values of $\theta$. This evaluation is repeated 100 times throughout the entire training process to monitor the progression and stability of the training.

All experiments were conducted on a single NVIDIA RTX 6000 Ada GPU. Due to variations in the total number of training steps across different datasets, the running time also varies. For illustration, we consider a typical NN Small dataset with 1,000 training steps. The running time for ETD is approximately 1 minute while DGI-full requires around 3 minutes. Since the result is averaged over 50 datasets, ETD takes a total of 50 minutes while DGI-full requires 2.5 hours. The evaluation phase for each trained model takes approximately 3 minutes for both methods.

\section{Complete and Enlarged Experimental Results.}
\label{sec:app-full-results}
\subsection{Performance on Optimization Tasks}
\begin{table}[h]
\caption{Performance for Scarce-Data Regime. \textbf{\emph{Lower}} value indicates \textbf{\emph{better}} performance.} \label{tab:scarce}
\begin{center}
\begin{tabular}{lllll}
\textbf{}  &\textbf{Linear} &\textbf{NN-s} &\textbf{Zak} &\textbf{Perm} \\
\hline 
RS  &0.7559&1.0850&0.4161&2.6436 \\
OC  &0.3932&0.5587&0.0248&0.5083 \\
\hdashline
ETD(R) &0.7085&1.0081&0.2174&0.6337 \\
ETD(G)  &0.6949&0.9393&0.2174&0.6238 \\
DGI(R) &0.5661&0.8016&0.0745&0.3267 \\
DGI(G) &0.5593&0.7854&0.0745&0.3267 \\
% SGDRS        &0.295&0.247&0.0161&0.0303\\
% RS           &0.518&0.515&0.0228&0.1104\\
% OC           &0.411&0.385&0.0165&0.0457\\
% ETD(R)       &0.504&0.496&0.0196&0.0495\\
% ETD(G)       &0.500&0.479&0.0196&0.0492\\
% DGI(R)       &0.462&0.445&0.0173&0.0402\\
% DGI(G)       &0.460&0.441&0.0173&0.0402

\end{tabular}
\end{center}
\end{table}

\begin{table}[h]
\caption{Performance for Scarce-Data Regime. \textbf{\emph{Lower}} value indicates \textbf{\emph{better}} performance.} \label{tab:scarce}
\begin{center}
\begin{tabular}{lllll}
\textbf{}  &\textbf{Rose} &\textbf{NN-l} &\textbf{Trid} &\textbf{Dix} \\
\hline 
RS  &1.2697&0.8808&0.8621&0.4379 \\
OC  &0.1798&0.4885&0.2184&0.0784 \\
\hdashline
ETD(R) &0.7978&0.8154&0.1034&0.0654 \\
ETD(G)  &0.8539&0.8846&0.1092&0.0719 \\
DGI(R) &0.5506&0.8885&0.1034&0.0654 \\
DGI(G) &0.5506&0.8885&0.1034&0.0654 \\
% SGDRS     &0.089&0.260&0.174&0.153\\
% RS        &0.202&0.489&0.324&0.220\\
% OC        &0.105&0.387&0.212&0.165\\
% ETD(R)    &0.160&0.472&0.192&0.163\\
% ETD(G)    &0.165&0.490&0.193&0.164\\
% DGI(R)    &0.138&0.491&0.192&0.163\\
% DGI(G)    &0.138&0.491&0.192&0.163

\end{tabular}
\end{center}
\end{table}

\begin{table}[h]
\caption{Performance for Large-Data Regime. Tasks include Quadratic (\textbf{\emph{QR}}), Ackley (\textbf{\emph{AK}}), Beam (\textbf{\emph{BM}}), and Vessel (\textbf{\emph{VL}}). \textbf{\emph{Lower}} value indicates \textbf{\emph{better}} performance.} 
\label{tab:large}
\begin{center}
\begin{tabular}{rrrrrr}
\textbf{}  &\textbf{QR}  &\textbf{AK}&\textbf{BM} &\textbf{VL} \\
\hline 
% SGDRS        &0     &0.541&-0.0005&0\\
% RS           &0.0815&0.761&0.215&0.185\\
% OC           &0.00071&0.380&0.003&0.0001\\
% ETD(R)      &0.0208 &0.558&0.038&0.0275\\
% ETD(G)       & 0.0191&0.382&0.042&0.0264\\
% DGI(R)      & 0.0032&0.446&-0.0005&0.0219\\
% DGI(G)       &0.0027&0.375&-0.0005&0.0222

RS  &814.0&0.4067&431.0&1849 \\
OC  &6.1&-0.2976&7.0&0.0 \\
\hdashline
ETD(R) &207.0&0.0314&77.0&274.0 \\
ETD(G)  &190.0&-0.2939&85.0&263.0 \\
DGI(R) &31.0&-0.1756&0.0&218.0 \\
DGI(G) &26.0&-0.3068&0.0&221.0 \\

\end{tabular}
\end{center}
\end{table}

% \begin{table}[h]
% \caption{Performance for Scarce-Data Regime} \label{sample-table}
% \begin{center}
% \begin{tabular}{lllll}
% \textbf{}  &\textbf{Linear} &\textbf{NN-s} &\textbf{Zak} &\textbf{Perm} \\
% \hline \\
% RS(ood)        & 0.518&0.515 &0.0228& 0.1104\\
% RS             & 0.518&0.528 & 0.0240&0.1067\\
% OC           &0.411 &0.385&0.0165 &0.0457\\
% True SGD(RS) &0.295&0.247&0.0161&0.0303\\
% True SGD(greedy) &0.295&0.249&0.0161&0.0303\\
% ETD(RS)            &0.504 &0.496 &0.0196&0.0495\\
% ETD(greedy)             & 0.500&0.479 &0.0196&0.0492\\
% DGI(RS)&0.462 & 0.445&0.0173&0.0402\\
% DGI(greedy) & 0.460&0.441&0.0173&0.0402
% \end{tabular}
% \end{center}
% \end{table}

% \begin{table}[h]
% \caption{Performance for Scarce-Data Regime} \label{sample-table}
% \begin{center}
% \begin{tabular}{lllll}
% \textbf{}  &\textbf{Rose} &\textbf{NN-l} &\textbf{Trid} &\textbf{Dix} \\
% \hline \\
% RS(ood)        & 0.202& 0.489&0.324& 0.220\\
% RS             &0.203 &0.489 & 0.299&0.219\\
% OC           &0.105 & 0.387&0.212 &0.165\\
% True sgd (RS) &0.089&0.260&0.174&0.153\\
% True sgd (greedy)&0.0891&0.263&0.174&0.153\\
% ETD(RS)            &0.160 & 0.472&0.192&0.163\\
% ETD(greedy)             &0.165 &0.490 &0.193&0.164\\
% DGI(RS)&0.138&0.491 &0.192&0.163\\
% DGI(greedy) &0.138& 0.491&0.192&0.163
% \end{tabular}
% \end{center}
% \end{table}

\subsection{Gradient Estimation}
\subsubsection{Simulation Problem Results in the Scarce-Data Regime}
\begin{figure}[!htbp]
\begin{subfigure}[b]{\textwidth}
       \centering
% Row 1
    \includegraphics[scale=0.6]
    {AISTATS2025/images/legend_row.png}
\end{subfigure}
\begin{subfigure}[b]{\textwidth}
\centering
% Row 1
\includegraphics[width=\textwidth]
{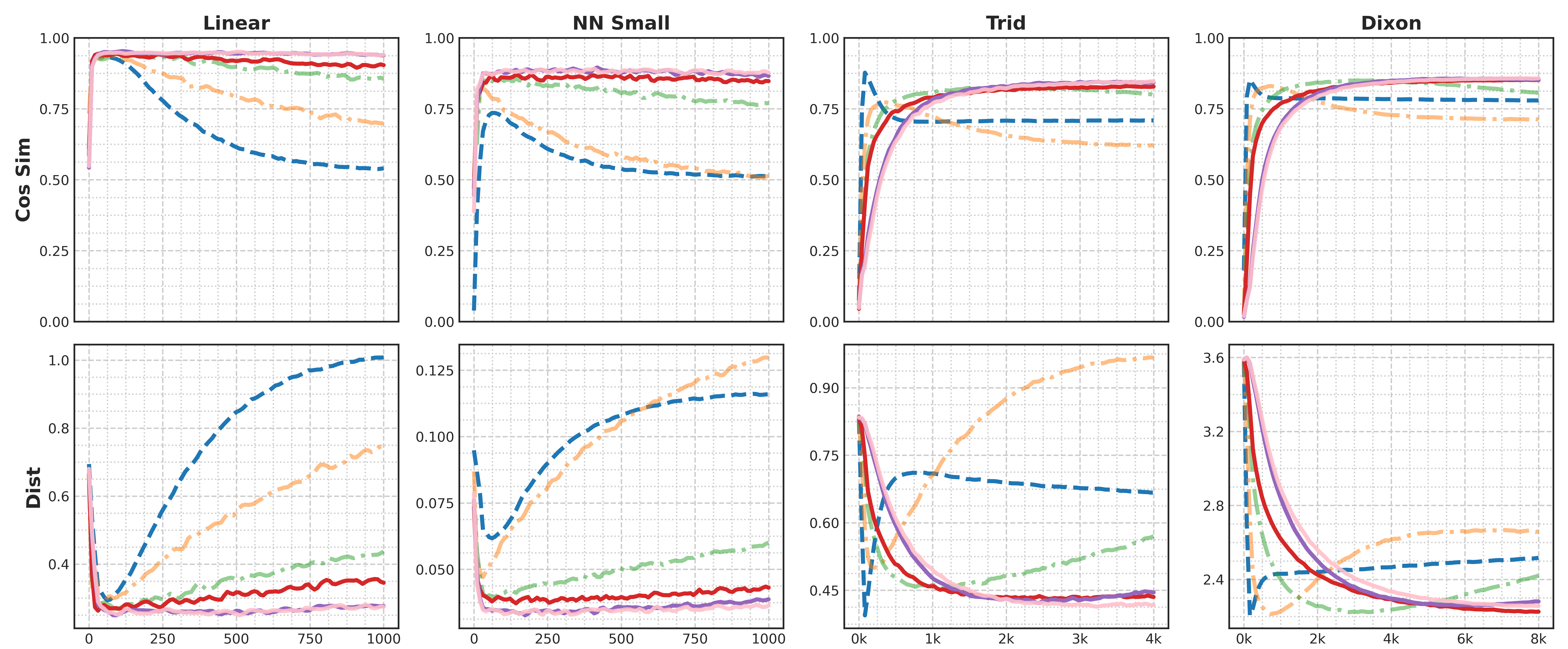}
\vspace{-0.5cm}
\end{subfigure}
\begin{subfigure}[b]{\textwidth}
    \centering
% Row 2
    \includegraphics[width=\textwidth]
    {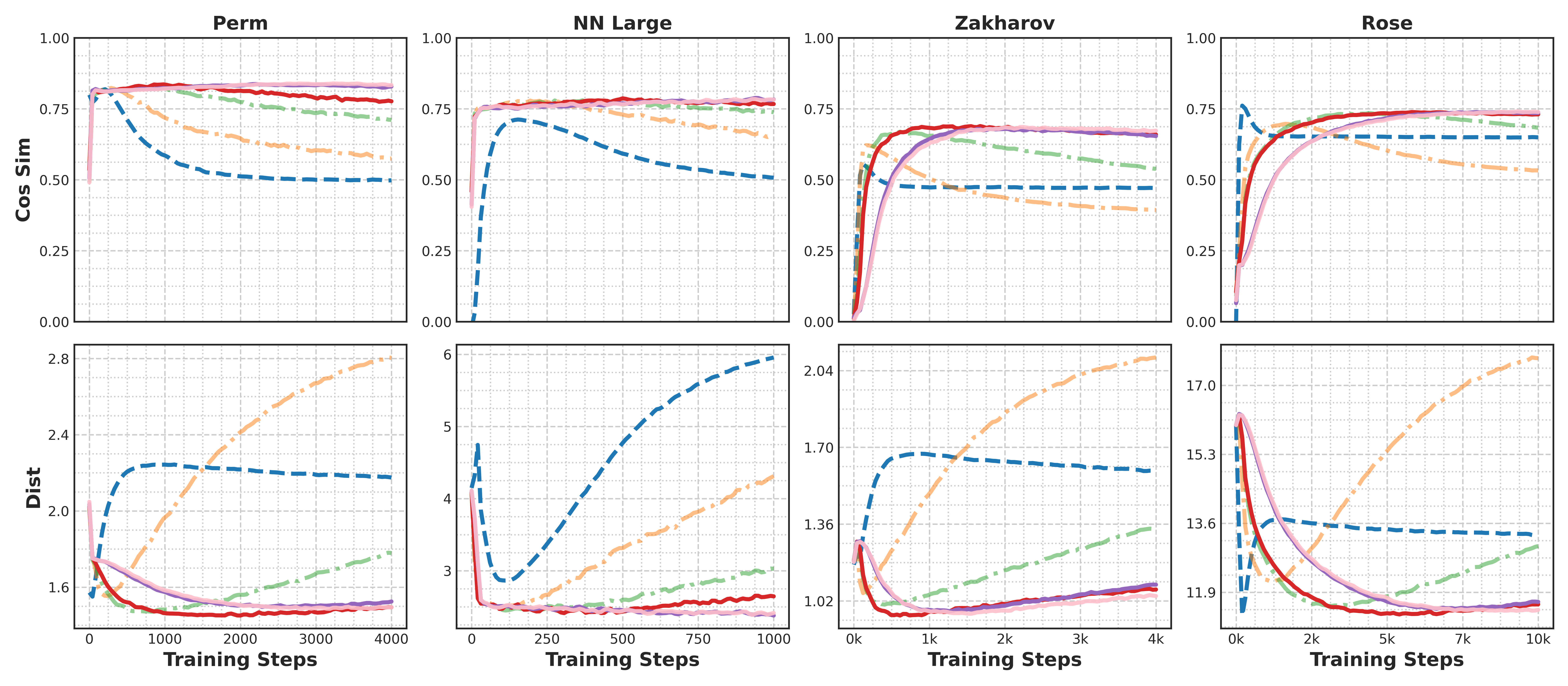}
\end{subfigure}
\caption{Results of eight simulation problems in the scarce-data regime. Top figures in each row show cosine similarity (higher is better); bottom figures show gradient norm distance (lower is better).}
\end{figure}

% \begin{figure}[!htbp]
%         \centering
% % Row 1
%     \includegraphics[scale=0.6]
%     {AISTATS2025/images/legend_row.png}
%     \vspace{-0.5cm}
% \end{figure}

% \begin{figure}[!htbp]
% \centering
% % Row 1
%     \includegraphics[width=\textwidth]
%     {NeurIPS2024/images3/Row1_cut.png}
%     \vspace{-0.5cm}
% \end{figure}

% \begin{figure}[!htbp]
% \centering
% % Row 2
%     \includegraphics[width=\textwidth]
%     {NeurIPS2024/images3/Row2.png}
%     \caption{Results of eight simulation problems in the scarce-data regime. Top figures in each row show cosine similarity (higher is better); bottom figures show gradient norm distance (lower is better).}
% \end{figure}

\subsubsection{Simulation Problem Results in the Large-Data Regime}
\begin{figure}[h]
\label{fig:sim-large}
\centering
\begin{subfigure}[b]{0.8\textwidth}
    \includegraphics[width=\textwidth]{AISTATS2025/images/legend_row.png}
    \label{fig:sub1}
    \vspace{-1em}
\end{subfigure}
\hfill
\\
% Row 1
\begin{subfigure}[b]{0.32\textwidth}
    \includegraphics[width=\textwidth]
    {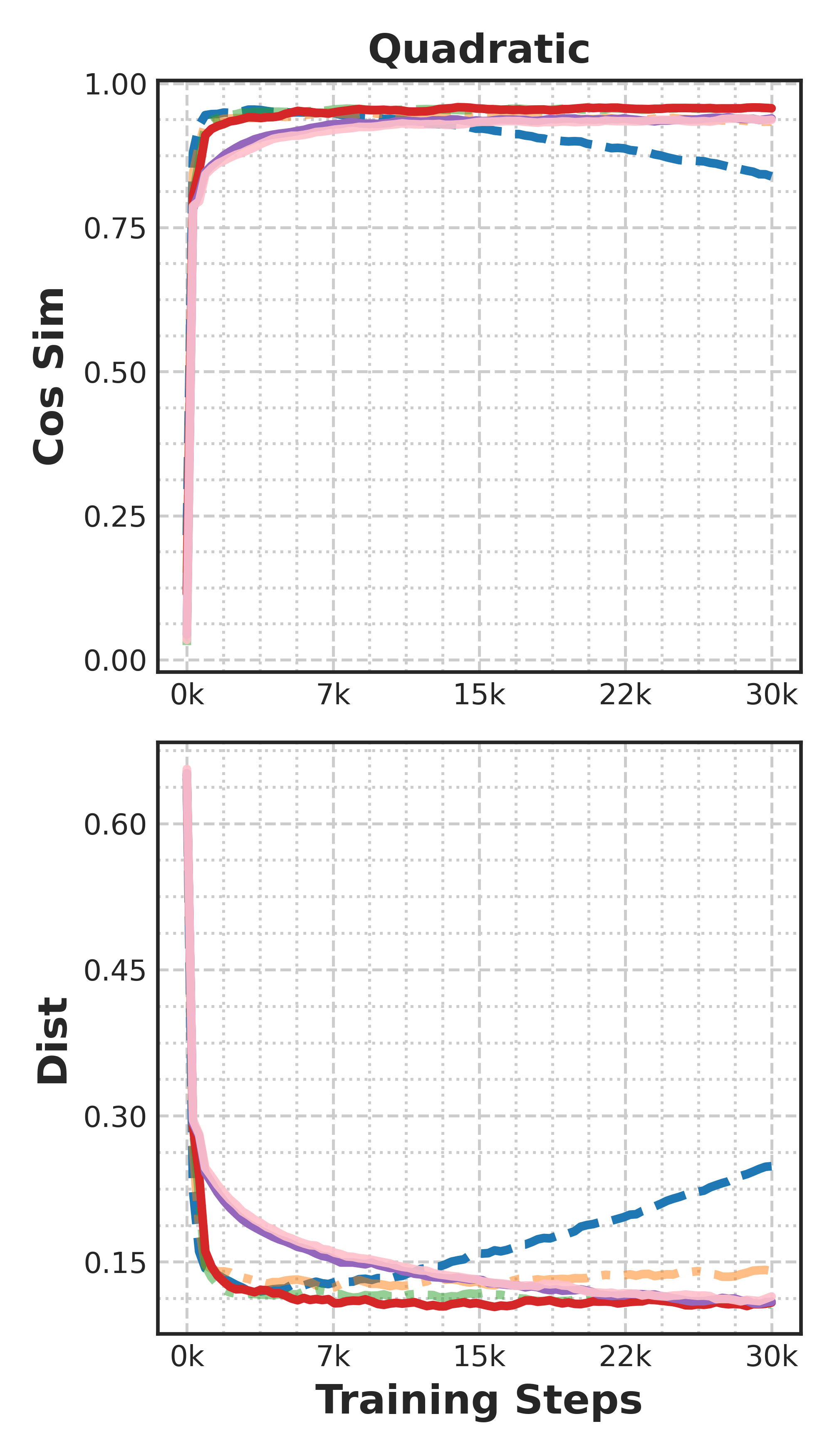}
    \label{fig:sub1}
\end{subfigure}
\hfill
\begin{subfigure}[b]{0.32\textwidth}
    \includegraphics[width=\textwidth]
    {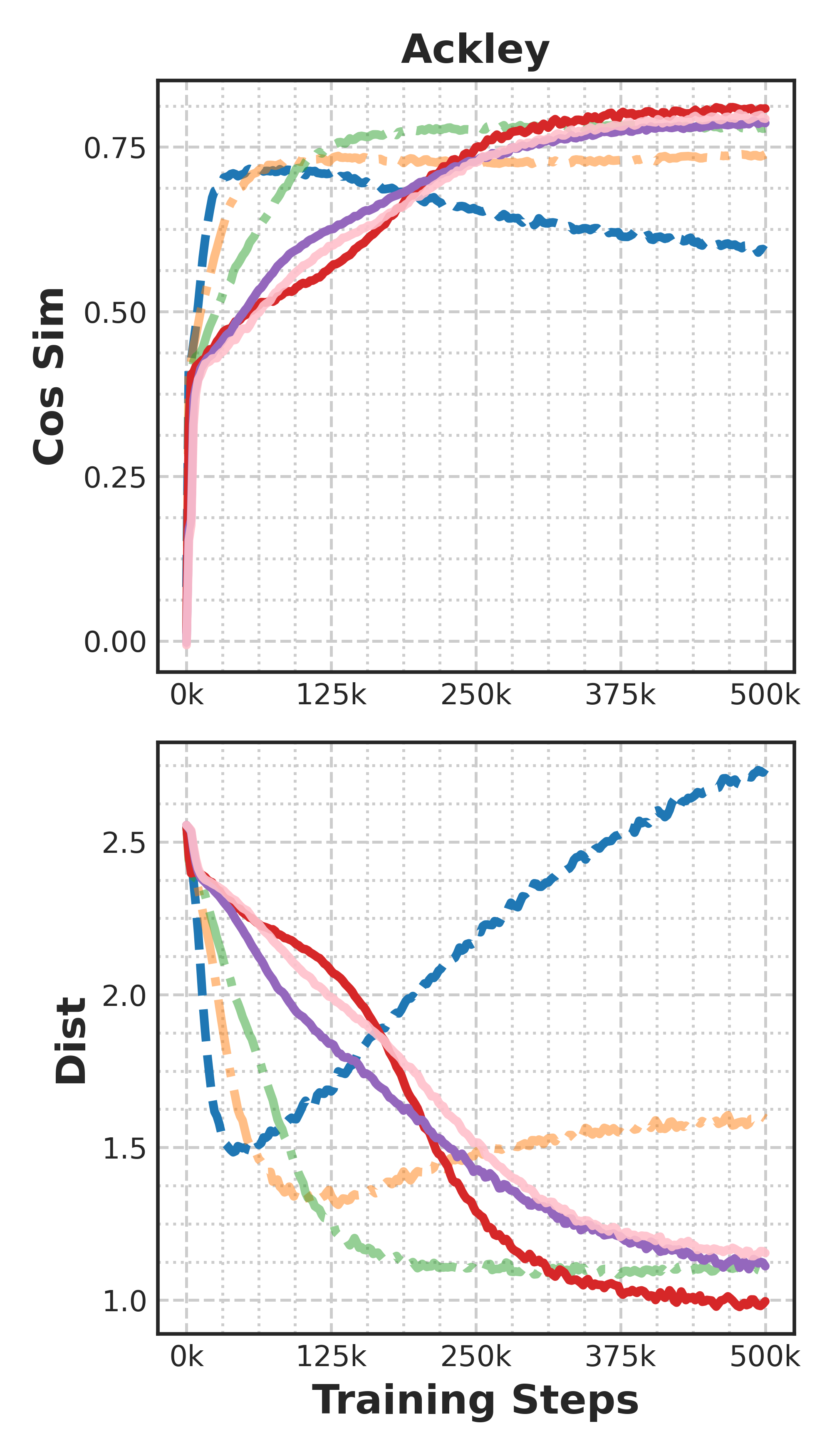}
    \label{fig:sub1}
\end{subfigure}
\hfill
\begin{subfigure}[b]{0.32\textwidth}
    \includegraphics[width=\textwidth]
    {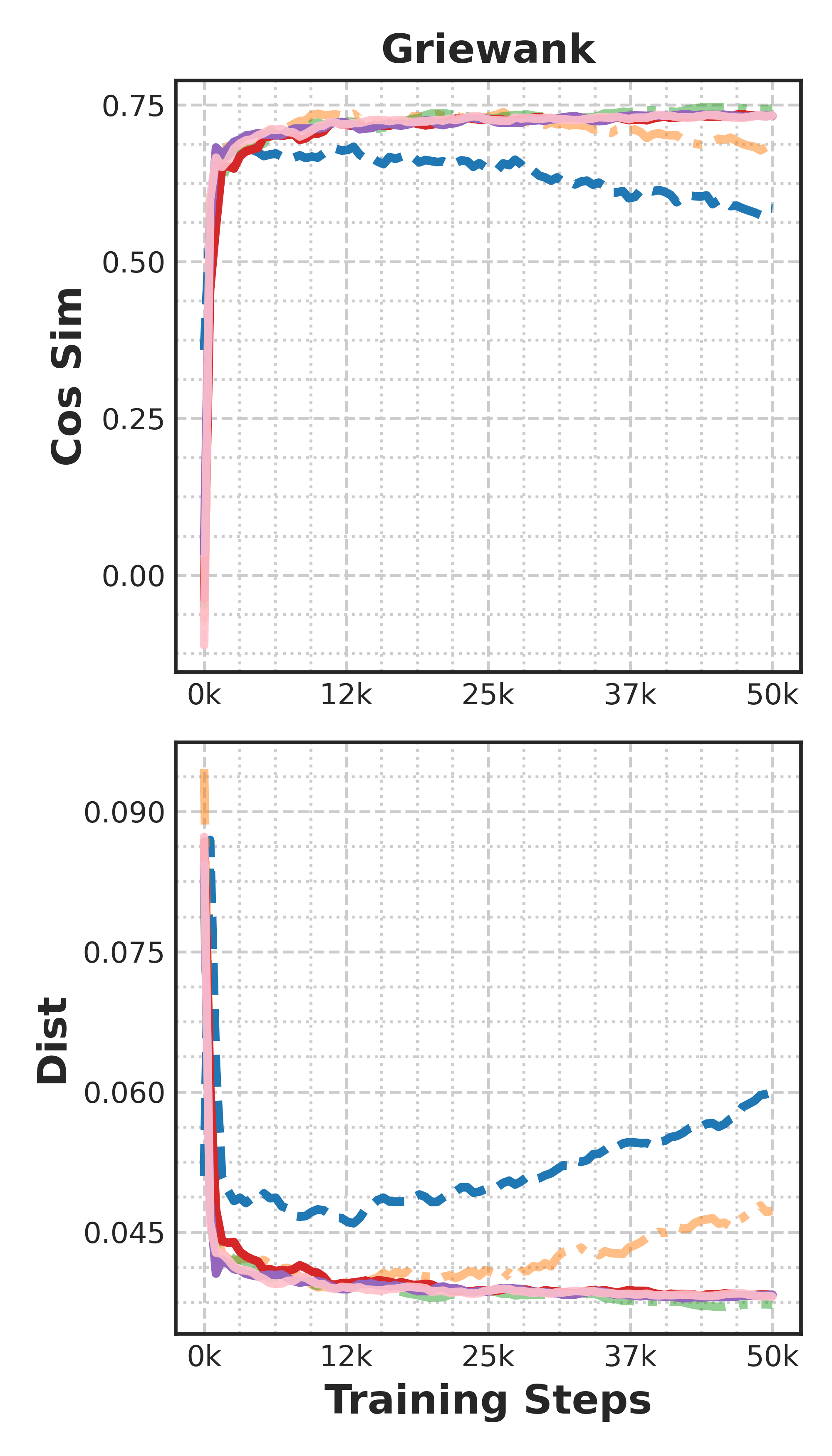}
    \label{fig:sub1}
\end{subfigure}
\vspace{-1em}
\caption{Results for simulation problems in the large-data regime. (Left two) Cosine similarity and norm distance for \textbf{Quadratic} objective functions. (Middle two) \textbf{Ackley}. (Right two) \textbf{Griewank}.}
\label{fig:sim-large}
\end{figure}

\subsubsection{Real World Application Results}
\label{sec:real-sim}
See Figure~\ref{fig:real-sim}.
\begin{figure}[h]
\centering
% Row 1
\begin{subfigure}[b]{0.8\textwidth}
    \includegraphics[width=\textwidth]{AISTATS2025/images/legend_row.png}
    \label{fig:sub5}
\end{subfigure}
\hfill
\\
% Row 2
\begin{subfigure}[b]{\textwidth}
    \includegraphics[width=\textwidth]
    {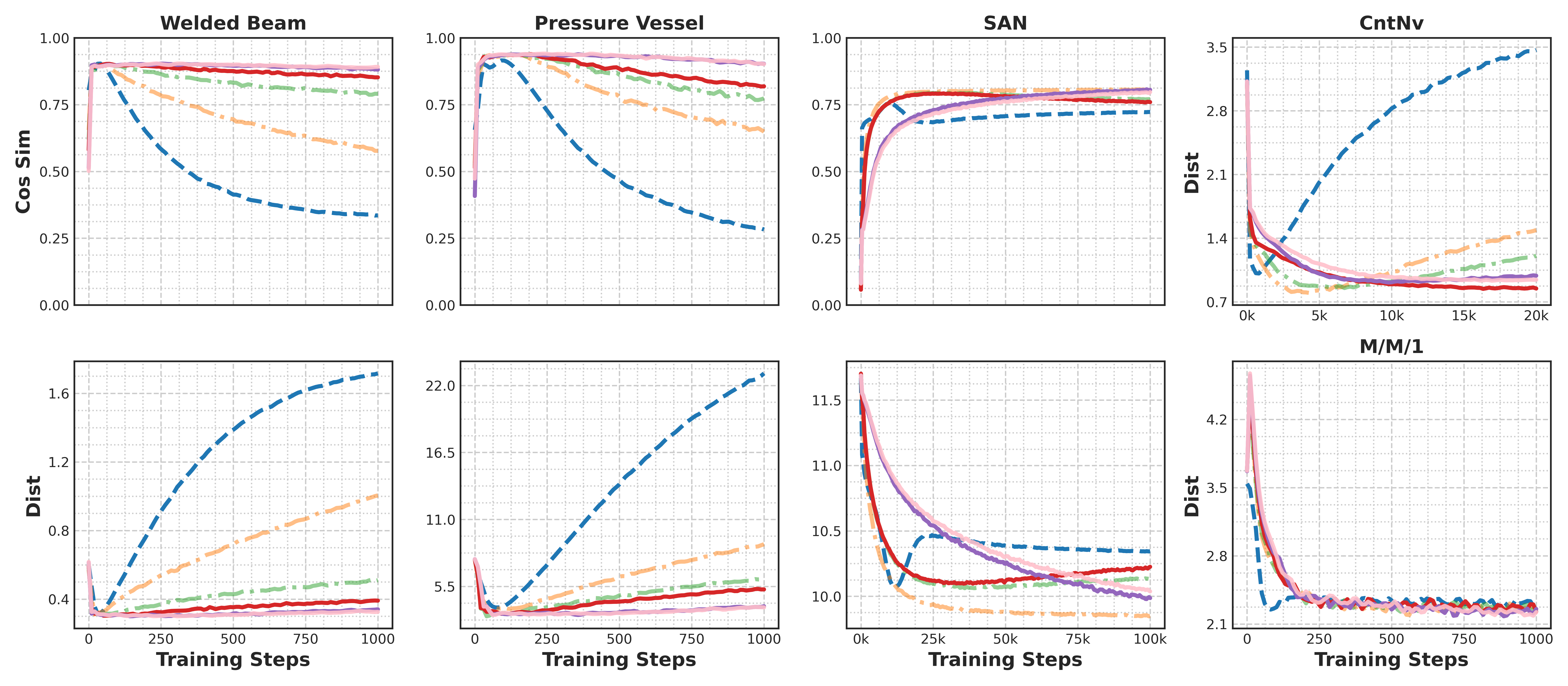}
    \label{fig:sub1}
\end{subfigure}
\caption{(From left to right) Cosine similarity and norm distance for \textbf{Welded Beam}, \textbf{Pressure Vessel}, \textbf{SAN}, and norm distance for \textbf{CntNv}, \textbf{M/M/1}.}
\label{fig:real-sim}
\end{figure}

\section{Extra Experimental Results}\label{sec:exp-more}
\subsection{Datasets Used in Ablation Studies}
We conduct all ablation studies using the NN small simulation problem in the scarce-data regime, except for the ablation on integration accuracy, which is performed on the Trid simulation problem to provide a clearer demonstration.
\subsection{Ablation on Noise}
In real-world scenarios, noise is ubiquitous and often substantial. DGI is particularly well-suited for noisy environments, as our experiments in Figure~\ref{fig:main-figure} demonstrate its robustness across various noise levels. We consider problems of varying single-sample SNR values in  $\{\infty\; \text{(no noise)}, 2, 1, 0.2, 0.1, 0.05\}$ and apply both ETD and DGI to these problems. Without noise, both methods perform well. However, with minimal noise (SNR = $2$ or $1$), the cosine similarity of ETD quickly drops to $0.7$, whereas DGI remains above $0.9$. Even under conditions with extreme noise (SNR = $0.05$), where the cosine similarity of ETD falls to $0.2$, DGI still maintains a cosine similarity of approximately $0.6$. 

\begin{figure}[H]
    \centering
    
    % First subfigure
    \begin{subfigure}[b]{0.45\textwidth}
        \centering
        \includegraphics[width=\textwidth]{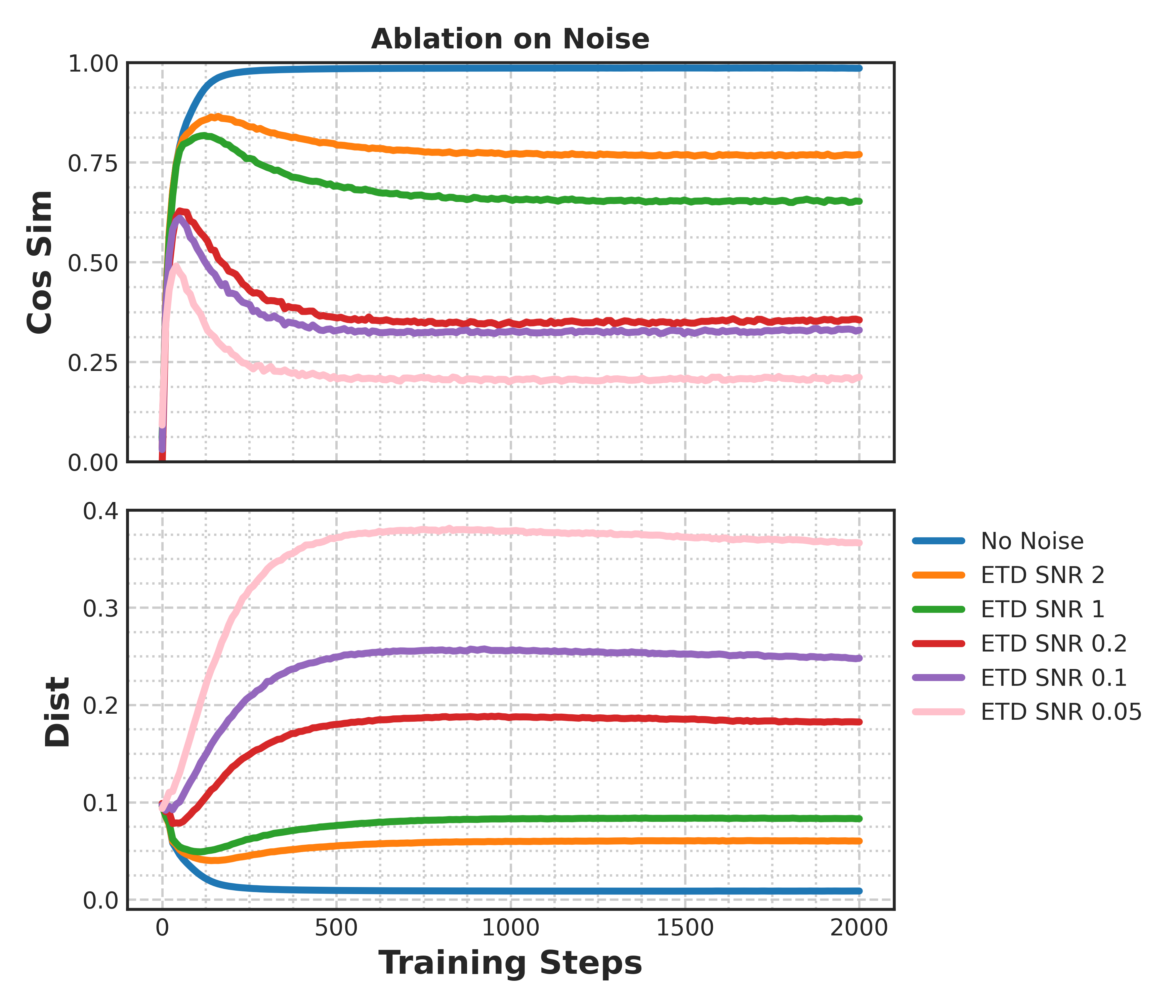}
        \caption{Ablation on Noise (ETD).}
        \label{fig:subfigure1}
    \end{subfigure}
    \hfill
    % Second subfigure
    \begin{subfigure}[b]{0.45\textwidth}
        \centering
        \includegraphics[width=\textwidth]{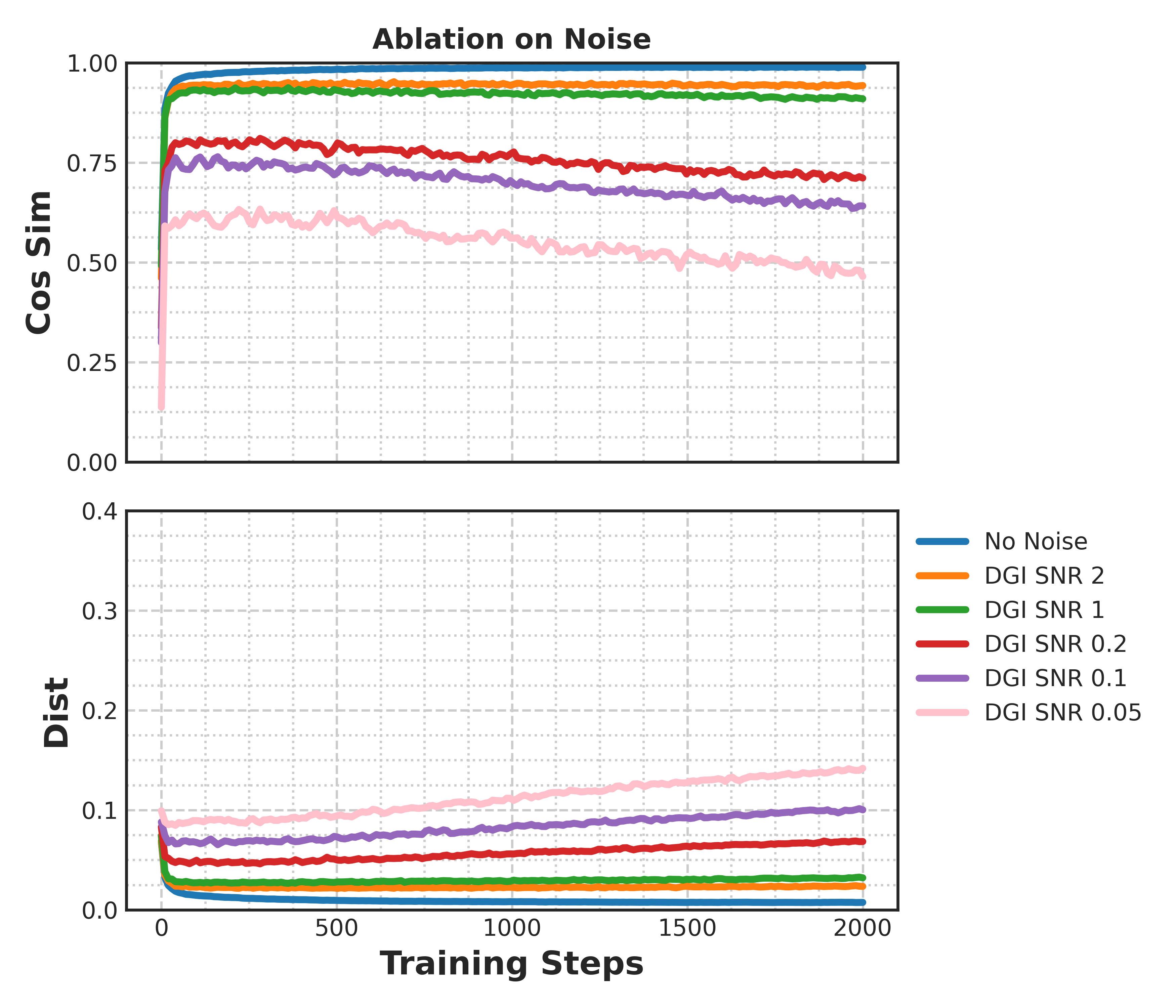}
        \caption{Ablation on Noise (DGI).}
        \label{fig:subfigure2}
    \end{subfigure}
    
    % Caption for the whole figure
    \caption{Ablation on Noise}
    \label{fig:main-figure}
\end{figure}

\section{Gradient Estimation and Train/Val Loss}
\label{sec:exp-overfit}
As shown in Figure~\ref{fig:loss-graph}, the convergence of the training loss for ETD does not necessarily correspond to optimal gradient estimation. Instead, it may indicate overfitting to noise, leading to degraded gradient estimation. In contrast, the convergence of the training loss for DGI-full corresponds to a much more stable convergence to optimal gradient estimation, highlighting the superiority of DGI methods over ETD.

\begin{figure}[H]
% Row 1
\begin{subfigure}[b]{0.49\textwidth}
    \includegraphics[width=\textwidth]{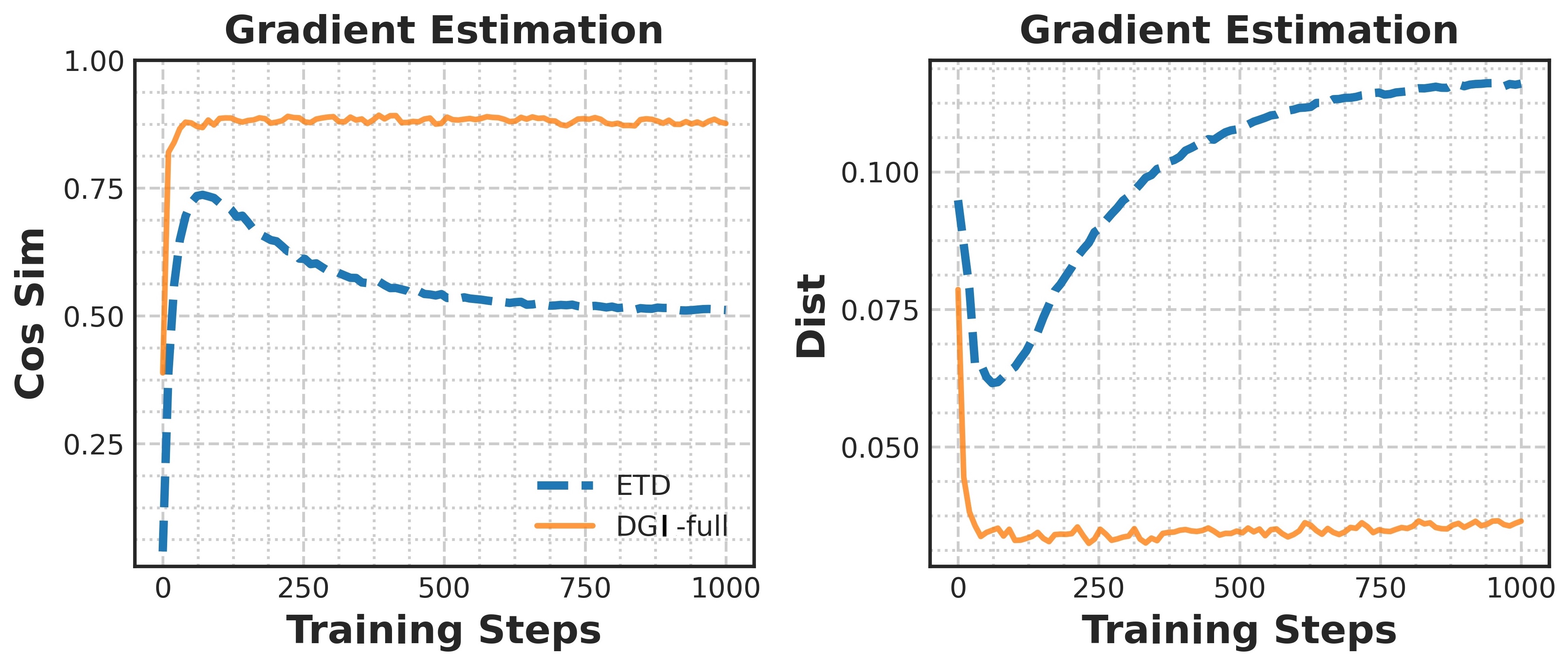}
    \label{fig:sub5}
\end{subfigure}
\hfill
% Row 2
\begin{subfigure}[b]{0.24\textwidth}
    \includegraphics[width=\textwidth]{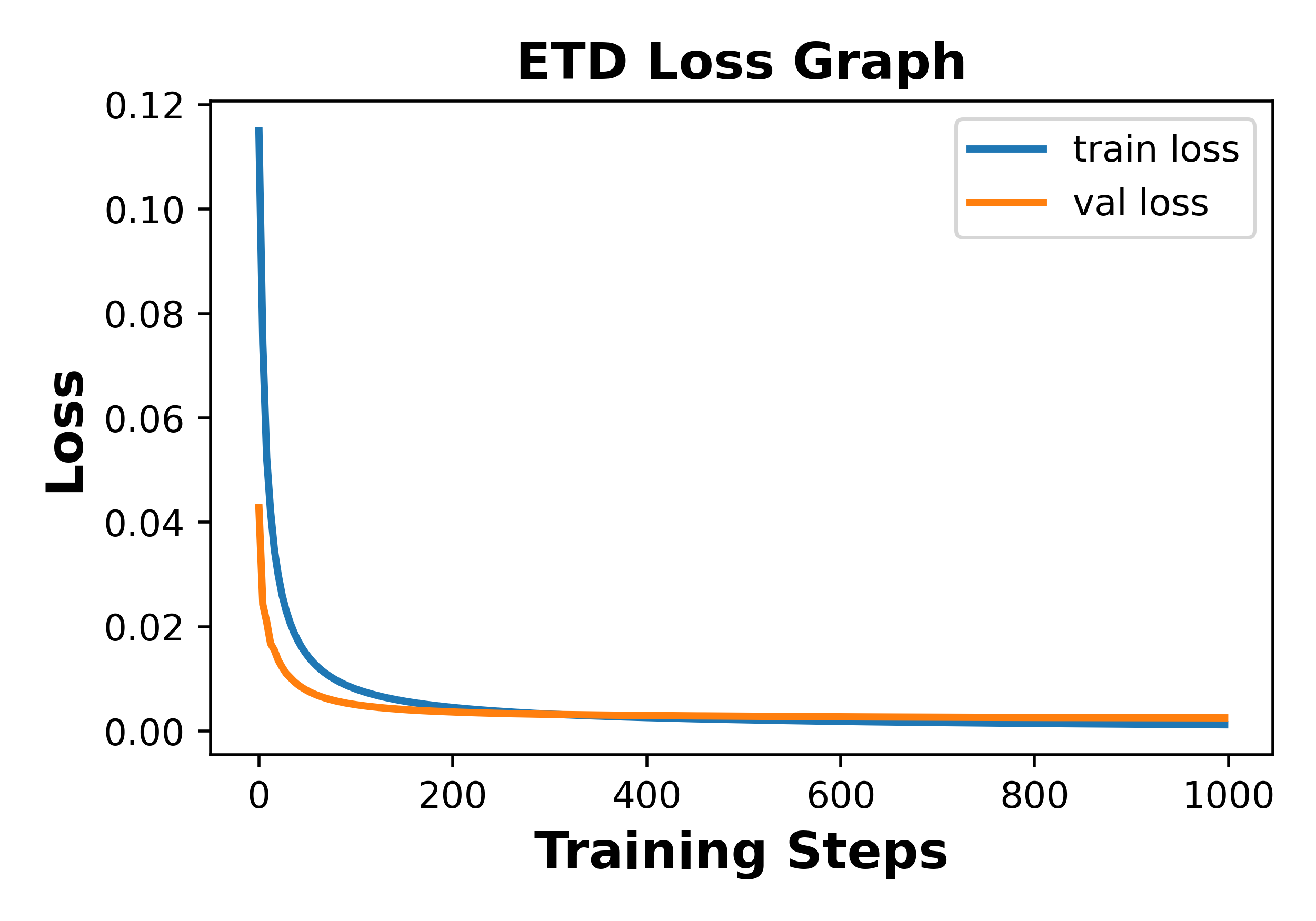}
    \label{fig:sub1}
\end{subfigure}
\hfill
\begin{subfigure}[b]{0.24\textwidth}
    \includegraphics[width=\textwidth]{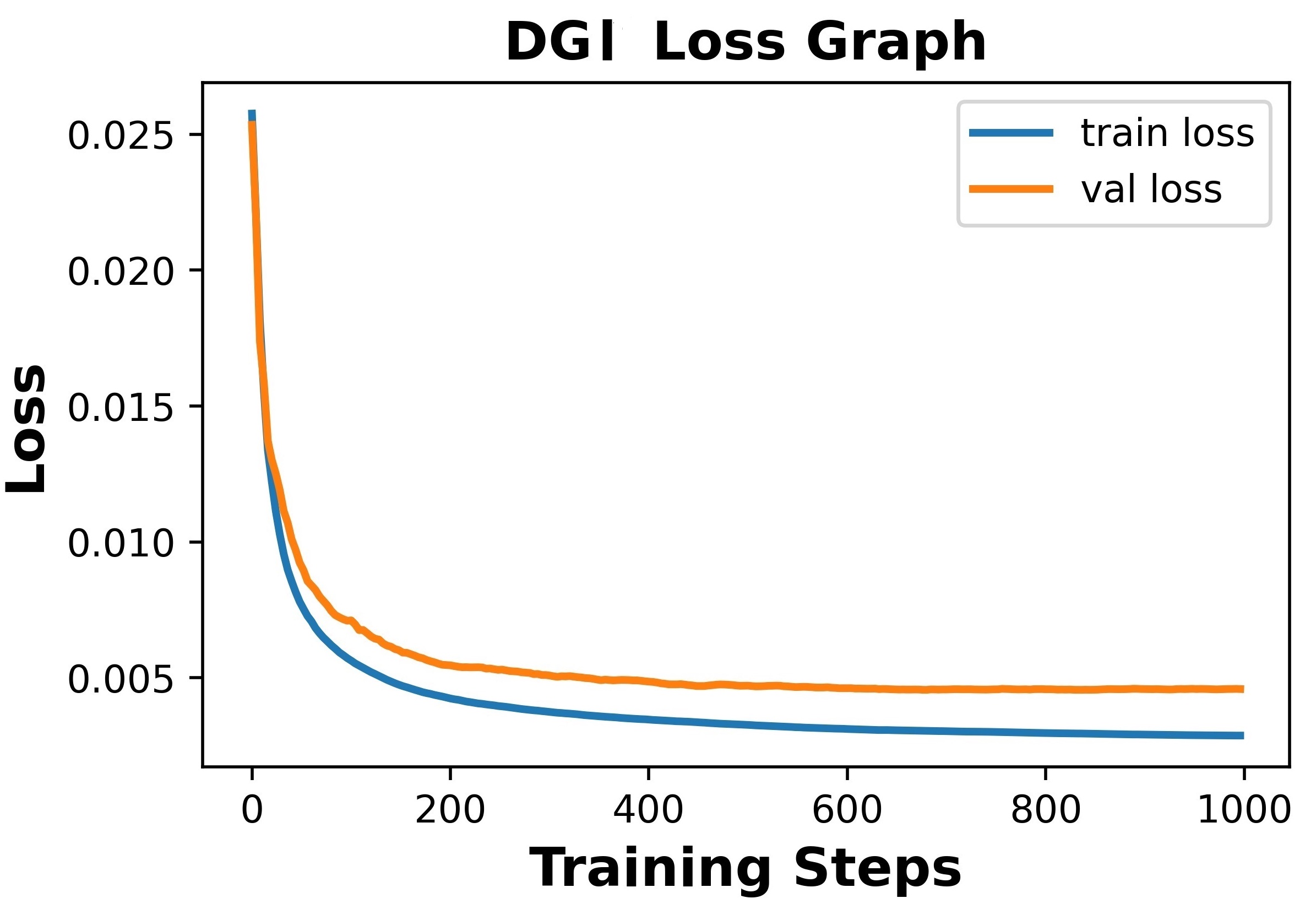}
    \label{fig:sub1}
\end{subfigure}
\caption{(From left to right) Cosine similarity and norm distance for NN small dataset in the scarce-data regime, train/val loss graph for ETD, and train/val loss graph for DGI-full.}
\label{fig:loss-graph}
\end{figure}

\section{Pseudocode}
% The notations are not consistent with the ones in the main text
\begin{algorithm}[H]
\caption{Deep Gradient Interpolation (DGI)}
\label{algo:dgx}
\begin{algorithmic}[1] % Enable line numbering
\State \textbf{Input:} A neural network $h$ with parameters $\phi$, training dataset $D$, number of paths sampled $\Lambda$, batch size $B$  for estimation/path independence loss, $M$ for balance losses, $\alpha$ for balance weight, $\eta$ for learning rate, and $T$ for the number of iterations.
\State Randomly initialize $\phi$.
\For{$T$ iterations}
    \State Sample batch  $\{(\theta_i, x_i, y_i)\}_{i = 1,...,B} \sim D$
    \State $\zeta_i\gets \text{concat}(\theta_i, x_i)$ for all $i \in  \{1,\dots,B\}$
    \For{$i,j \in \{1,\dots,B\}$}
        \State $\{r^{i,j}_\lambda\}^\Lambda_{\lambda=1}\leftarrow\;\;\text{sample $\Lambda$ paths between $\zeta_i$ and $\zeta_j$}$
    \EndFor
    \State Estimation Loss and path independence loss:
    $$
     \widehat{L_e}(\phi) = \frac{1}{B^2}\sum_{k,k' \in [B]} \max_{r \in \{r^{k,k'}_\lambda\}^\Lambda_{\lambda=1}}\left( (y_k-y_{k'})-\int_{r }h_\phi(\zeta)d\zeta \right)^2.
    $$
    \State Uniformly sample pairs of indices $\{(k_m,l_m)\}_{m = 1,...,M} \sim [d]\times[d], k\neq l$  
    \State Balance Loss: 
    $$
     \widehat{L_b}(\phi) = \frac{1}{BM}\sum_{k=1}^B\sum_{m=1}^M\left(\frac{\partial h_\phi^{k_m}(\zeta_k)}{\partial \zeta^{l_m}} -\frac{\partial h_\phi^{l_m}(\zeta_k)}{\partial \zeta^{k_m}}\right)^2.
    $$
    \State $L(\phi) \gets  \widehat{L_e}(\phi) + \alpha\cdot \widehat{L_b}(\phi)$
    \State $\phi \gets \phi - \eta \cdot \nabla_\phi L(\phi)$
\EndFor
\State \textbf{return} $\widehat{\phi}=\phi$
\end{algorithmic}
\end{algorithm}


\begin{thebibliography}{}

\bibitem[Arora et~al., 2018]{arora2018convergence}
Arora, S., Cohen, N., Golowich, N., and Hu, W. (2018).
\newblock A convergence analysis of gradient descent for deep linear neural networks.
\newblock {\em arXiv preprint arXiv:1810.02281}.

\bibitem[ASO and Yorozu, 1991]{aso1991generalization}
ASO, K. and Yorozu, S. (1991).
\newblock A generalization of clairaut's theorem and umbilical foliations.

\bibitem[Avramidis and Wilson, 1996]{avramidis1996integrated}
Avramidis, A.~N. and Wilson, J.~R. (1996).
\newblock Integrated variance reduction strategies for simulation.
\newblock {\em Operations Research}, 44(2):327--346.

\bibitem[Bertsimas and Tsitsiklis, 1993]{bertsimas1993simulated}
Bertsimas, D. and Tsitsiklis, J. (1993).
\newblock Simulated annealing.
\newblock {\em Statistical science}, 8(1):10--15.

\bibitem[Bottou, 2010]{bottou2010large}
Bottou, L. (2010).
\newblock Large-scale machine learning with stochastic gradient descent.
\newblock In {\em Proceedings of COMPSTAT'2010: 19th International Conference on Computational StatisticsParis France, August 22-27, 2010 Keynote, Invited and Contributed Papers}, pages 177--186. Springer.

\bibitem[Brooks, 1958]{brooks1958discussion}
Brooks, S.~H. (1958).
\newblock A discussion of random methods for seeking maxima.
\newblock {\em Operations research}, 6(2):244--251.

\bibitem[Chemingui et~al., 2024]{chemingui2024offline}
Chemingui, Y., Deshwal, A., Hoang, T.~N., and Doppa, J.~R. (2024).
\newblock Offline model-based optimization via policy-guided gradient search.
\newblock In {\em Proceedings of the AAAI Conference on Artificial Intelligence}, volume~38, pages 11230--11239.

\bibitem[Cheng and Kleijnen, 1999]{cheng1999improved}
Cheng, R.~C. and Kleijnen, J.~P. (1999).
\newblock Improved design of queueing simulation experiments with highly heteroscedastic responses.
\newblock {\em Operations research}, 47(5):762--777.

\bibitem[Daoud et~al., 2023]{daoud2023gradient}
Daoud, M.~S., Shehab, M., Al-Mimi, H.~M., Abualigah, L., Zitar, R.~A., and Shambour, M. K.~Y. (2023).
\newblock Gradient-based optimizer (gbo): a review, theory, variants, and applications.
\newblock {\em Archives of Computational Methods in Engineering}, 30(4):2431--2449.

\bibitem[Deb, 1991]{deb1991optimal}
Deb, K. (1991).
\newblock Optimal design of a welded beam via genetic algorithms.
\newblock {\em AIAA journal}, 29(11):2013--2015.

\bibitem[Duchi et~al., 2011]{duchi2011adaptive}
Duchi, J., Hazan, E., and Singer, Y. (2011).
\newblock Adaptive subgradient methods for online learning and stochastic optimization.
\newblock {\em Journal of machine learning research}, 12(7).

\bibitem[Eckman et~al., 2023]{simoptgithub}
Eckman, D.~J., Henderson, S.~G., Shashaani, S., and Pasupathy, R. (2023).
\newblock {SimOpt}.
\newblock \url{https://github.com/simopt-admin/simopt}.

\bibitem[Fouskakis and Draper, 2002]{stochastic_opt_review}
Fouskakis, D. and Draper, D. (2002).
\newblock Stochastic optimization: a review.
\newblock {\em International Statistical Review}, 70(3):315--349.

\bibitem[Glasserman, 1990]{glasserman1990gradient}
Glasserman, P. (1990).
\newblock {\em Gradient estimation via perturbation analysis}, volume 116.
\newblock Springer Science \& Business Media.

\bibitem[Glover, 1990]{glover1990tabu}
Glover, F. (1990).
\newblock Tabu search: A tutorial.
\newblock {\em Interfaces}, 20(4):74--94.

\bibitem[Glover, 1999]{glover1999scatter}
Glover, F. (1999).
\newblock Scatter search and path relinking.
\newblock {\em New ideas in optimization}, 138.

\bibitem[Haji and Abdulazeez, 2021]{haji2021comparison}
Haji, S.~H. and Abdulazeez, A.~M. (2021).
\newblock Comparison of optimization techniques based on gradient descent algorithm: A review.
\newblock {\em PalArch's Journal of Archaeology of Egypt/Egyptology}, 18(4):2715--2743.

\bibitem[Hansen, 2000]{hansen2000sample}
Hansen, B.~E. (2000).
\newblock Sample splitting and threshold estimation.
\newblock {\em Econometrica}, 68(3):575--603.

\bibitem[Hinton et~al., 2012]{hinton2012neural}
Hinton, G., Srivastava, N., and Swersky, K. (2012).
\newblock Neural networks for machine learning lecture 6a overview of mini-batch gradient descent.
\newblock {\em Cited on}, 14(8):2.

\bibitem[Ho and Cao, 2012]{ho2012perturbation}
Ho, Y.-C.~L. and Cao, X.-R. (2012).
\newblock {\em Perturbation analysis of discrete event dynamic systems}, volume 145.
\newblock Springer Science \& Business Media.

\bibitem[Holland, 1992]{holland1992adaptation}
Holland, J.~H. (1992).
\newblock {\em Adaptation in natural and artificial systems: an introductory analysis with applications to biology, control, and artificial intelligence}.
\newblock MIT press.

\bibitem[Hu and Fu, 2024]{hu2024convergence}
Hu, J. and Fu, M.~C. (2024).
\newblock On the convergence rate of stochastic approximation for gradient-based stochastic optimization.
\newblock {\em Operations Research}.

\bibitem[Khouja, 1999]{khouja1999single}
Khouja, M. (1999).
\newblock The single-period (news-vendor) problem: literature review and suggestions for future research.
\newblock {\em omega}, 27(5):537--553.

\bibitem[Kim et~al., 2015]{kim2015guide}
Kim, S., Pasupathy, R., and Henderson, S.~G. (2015).
\newblock A guide to sample average approximation.
\newblock {\em Handbook of simulation optimization}, pages 207--243.

\bibitem[Kingma and Ba, 2014]{kingma2014adam}
Kingma, D.~P. and Ba, J. (2014).
\newblock Adam: A method for stochastic optimization.
\newblock {\em arXiv preprint arXiv:1412.6980}.

\bibitem[Lei and Jordan, 2020]{lei2020adaptivity}
Lei, L. and Jordan, M.~I. (2020).
\newblock On the adaptivity of stochastic gradient-based optimization.
\newblock {\em SIAM Journal on Optimization}, 30(2):1473--1500.

\bibitem[Liu et~al., 2020]{liu2020improved}
Liu, Y., Gao, Y., and Yin, W. (2020).
\newblock An improved analysis of stochastic gradient descent with momentum.
\newblock {\em Advances in Neural Information Processing Systems}, 33:18261--18271.

\bibitem[Maryak and Chin, 2001]{maryak2001global}
Maryak, J.~L. and Chin, D.~C. (2001).
\newblock Global random optimization by simultaneous perturbation stochastic approximation.
\newblock In {\em Proceedings of the 2001 American control conference.(Cat. No. 01CH37148)}, volume~2, pages 756--762. IEEE.

\bibitem[Mashkaria et~al., 2023]{mashkaria2023generative}
Mashkaria, S.~M., Krishnamoorthy, S., and Grover, A. (2023).
\newblock Generative pretraining for black-box optimization.
\newblock In {\em International Conference on Machine Learning}, pages 24173--24197. PMLR.

\bibitem[Moss, 2004]{moss2004pressure}
Moss, D.~R. (2004).
\newblock {\em Pressure vessel design manual}.
\newblock Elsevier.

\bibitem[Nesterov, 1983]{nesterov1983method}
Nesterov, Y. (1983).
\newblock A method of solving a convex programming problem with convergence rate o (1/k** 2).
\newblock {\em Doklady Akademii Nauk SSSR}, 269(3):543.

\bibitem[Paszke et~al., 2019]{paszke2019pytorch}
Paszke, A., Gross, S., Massa, F., Lerer, A., Bradbury, J., Chanan, G., Killeen, T., Lin, Z., Gimelshein, N., Antiga, L., et~al. (2019).
\newblock Pytorch: An imperative style, high-performance deep learning library.
\newblock {\em Advances in neural information processing systems}, 32.

\bibitem[Polyak, 1964]{polyak1964some}
Polyak, B.~T. (1964).
\newblock Some methods of speeding up the convergence of iteration methods.
\newblock {\em Ussr computational mathematics and mathematical physics}, 4(5):1--17.

\bibitem[Pronzato et~al., 1984]{pronzato1984general}
Pronzato, L., Walter, E., Venot, A., and Lebruchec, J.-F. (1984).
\newblock A general-purpose global optimizer: Implimentation and applications.
\newblock {\em Mathematics and Computers in Simulation}, 26(5):412--422.

\bibitem[Reeves, 1997]{reeves1997genetic}
Reeves, C.~R. (1997).
\newblock Genetic algorithms for the operations researcher.
\newblock {\em INFORMS journal on computing}, 9(3):231--250.

\bibitem[R{\'e}nyi, 2007]{renyi2007probability}
R{\'e}nyi, A. (2007).
\newblock {\em Probability theory}.
\newblock Courier Corporation.

\bibitem[Rubinstein and Shapiro, 1993]{rubinstein1993discrete}
Rubinstein, R.~Y. and Shapiro, A. (1993).
\newblock {\em Discrete event systems: sensitivity analysis and stochastic optimization by the score function method}, volume~13.
\newblock Wiley.

\bibitem[Ruder, 2016]{ruder2016overview}
Ruder, S. (2016).
\newblock An overview of gradient descent optimization algorithms.
\newblock {\em arXiv preprint arXiv:1609.04747}.

\bibitem[Scarselli and Tsoi, 1998]{scarselli1998universal}
Scarselli, F. and Tsoi, A.~C. (1998).
\newblock Universal approximation using feedforward neural networks: A survey of some existing methods, and some new results.
\newblock {\em Neural networks}, 11(1):15--37.

\bibitem[Shapiro and Wardi, 1996]{shapiro1996convergence}
Shapiro, A. and Wardi, Y. (1996).
\newblock Convergence analysis of gradient descent stochastic algorithms.
\newblock {\em Journal of optimization theory and applications}, 91:439--454.

\bibitem[Shorack and Wellner, 2009]{shorack2009empirical}
Shorack, G.~R. and Wellner, J.~A. (2009).
\newblock {\em Empirical processes with applications to statistics}.
\newblock SIAM.

\bibitem[Skiscim and Golden, 1983]{skiscim1983optimization}
Skiscim, C.~C. and Golden, B.~L. (1983).
\newblock Optimization by simulated annealing: A preliminary computational study for the tsp.
\newblock Technical report, Institute of Electrical and Electronics Engineers (IEEE).

\bibitem[Sobczyk and S{\'a}nchez, 2011]{sobczyk2011fundamental}
Sobczyk, G. and S{\'a}nchez, O.~L. (2011).
\newblock Fundamental theorem of calculus.
\newblock {\em Advances in Applied Clifford Algebras}, 21:221--231.

\bibitem[Spall, 1992]{spall1992multivariate}
Spall, J.~C. (1992).
\newblock Multivariate stochastic approximation using a simultaneous perturbation gradient approximation.
\newblock {\em IEEE transactions on automatic control}, 37(3):332--341.

\bibitem[Spall, 2009]{spall2009feedback}
Spall, J.~C. (2009).
\newblock Feedback and weighting mechanisms for improving jacobian estimates in the adaptive simultaneous perturbation algorithm.
\newblock {\em IEEE Transactions on Automatic Control}, 54(6):1216--1229.

\bibitem[Tekin and Sabuncuoglu, 2004]{tekin2004simulation}
Tekin, E. and Sabuncuoglu, I. (2004).
\newblock Simulation optimization: A comprehensive review on theory and applications.
\newblock {\em IIE transactions}, 36(11):1067--1081.

\bibitem[Trabucco et~al., 2021]{coms}
Trabucco, B., Kumar, A., Geng, X., and Levine, S. (2021).
\newblock Conservative objective models for effective offline model-based optimization.
\newblock In Meila, M. and Zhang, T., editors, {\em Proceedings of the 38th International Conference on Machine Learning}, volume 139 of {\em Proceedings of Machine Learning Research}, pages 10358--10368. PMLR.

\bibitem[Vapnik, 1991]{vapnik1991principles}
Vapnik, V. (1991).
\newblock Principles of risk minimization for learning theory.
\newblock {\em Advances in neural information processing systems}, 4.

\bibitem[Wang and Spall, 2011]{wang2011discrete}
Wang, Q. and Spall, J.~C. (2011).
\newblock Discrete simultaneous perturbation stochastic approximation on loss function with noisy measurements.
\newblock In {\em Proceedings of the 2011 American Control Conference}, pages 4520--4525. IEEE.

\bibitem[Whitley, 1994]{whitley1994genetic}
Whitley, D. (1994).
\newblock A genetic algorithm tutorial.
\newblock {\em Statistics and computing}, 4:65--85.

\bibitem[Yazan and Talu, 2017]{yazan2017comparison}
Yazan, E. and Talu, M.~F. (2017).
\newblock Comparison of the stochastic gradient descent based optimization techniques.
\newblock In {\em 2017 International Artificial Intelligence and Data Processing Symposium (IDAP)}, pages 1--5. IEEE.

\end{thebibliography}
\end{document}